\documentclass[a4,11pt]{article}

\usepackage[american]{babel}
\usepackage{amsmath}
\usepackage{amssymb}
\usepackage{amsthm}
\usepackage{amstext}
\usepackage{amscd}
\usepackage{amsfonts}
\usepackage{booktabs}
\usepackage{enumerate}
\usepackage{graphicx}
\usepackage{latexsym}
\usepackage{mathrsfs}
\usepackage{mathtools}
\usepackage{cases}
\usepackage{verbatim}
\usepackage{bm}
\usepackage{cases}  
\usepackage{empheq}
\usepackage{graphicx}
\usepackage{hhline}
\usepackage{tikz}
\usepackage{CJKutf8}
\usepackage{comment}
\theoremstyle{plain}
\usepackage{algorithm}
\usepackage{algorithmic}
\usepackage{color}
\usepackage{titlesec}
\usepackage{hyperref}
\usepackage{multirow}
\usepackage{adjustbox}
\usepackage{bibunits}
\usepackage{natbib}

\usepackage{authblk}
\usepackage[top=3cm, bottom=3cm, left=2.5cm, right=2.5cm]{geometry}

\newtheorem{thm}{Theorem}[section]
\newtheorem{cor}[thm]{Corollary}
\newtheorem{lem}[thm]{Lemma}
\newtheorem{prop}[thm]{Proposition}

\theoremstyle{definition}
\newtheorem{rem}[thm]{Remark}
\newtheorem{dfn}[thm]{Definition}

\numberwithin{equation}{section}

\mathtoolsset{showonlyrefs=true}

\title{\bf Theoretical Error Analysis of Entropy Approximation\\ for Gaussian Mixtures}

\date{}

\author[*,1,a]{Takashi Furuya}
\author[,2,b]{Hiroyuki Kusumoto\thanks{Equal contribution.}}
\author[3]{\\Koichi Taniguchi}
\author[4]{Naoya Kanno}
\author[5]{Kazuma Suetake}

\affil[1]{{\small Education and Research Center for Mathematical and Data Science, Shimane University, Japan}}
\affil[a]{{\small Email: takashi.furuya0101@gmail.com}\vspace{3mm}}

\affil[2]{{\small Graduate School of Mathematics, Nagoya University, Japan}}
\affil[b]{{\small Email: kusumoto-108@outlook.com}\vspace{3mm}}

\affil[3]{{\small Mathematical and Systems Engineering, 
Shizuoka University, Japan}}

\affil[4]{{\small Graduate School of Biomedical Engineering, Tohoku University, Japan}}

\affil[5]{{\small AISIN SOFTWARE, Japan}}

\begin{document}

\maketitle

\begin{abstract}
Gaussian mixture distributions are commonly employed to represent general probability distributions. Despite the importance of using Gaussian mixtures for uncertainty estimation, the entropy of a Gaussian mixture cannot be calculated analytically. In this paper, we study the approximate entropy represented as the sum of the entropies of unimodal Gaussian distributions with mixing coefficients. This approximation is easy to calculate analytically regardless of dimension, but there is a lack of theoretical guarantees. 
We theoretically analyze the approximation error between the true and the approximate entropy to reveal when this approximation works effectively. 
This error is essentially controlled by how far apart each Gaussian component of the Gaussian mixture is. 
To measure such separation, we introduce the ratios of the distances between the means to the sum of the variances of each Gaussian component of the Gaussian mixture, and we reveal that the error converges to zero as the ratios tend to infinity. 
In addition, the probabilistic estimate indicates that this convergence situation is more likely to occur in higher-dimensional spaces. 
Therefore, our results provide a guarantee that this approximation works well for high-dimensional problems, such as neural networks that involve a large number of parameters.
\end{abstract}


\section{Introduction}\label{Introduction}

Entropy is a fundamental measure of uncertainty in information theory with many applications in machine learning, data compression, and image processing. 
In machine learning, for instance,  entropy is a key component in the evidence lower bound (ELBO) in the variational inference \citep{hinton1993keeping, barber1998ensemble, bishop2006PRML} and the variational autoencoder \citep{kingma2013auto}.
It also plays a crucial role in the data acquisition for Bayesian optimization \citep{frazier2018tutorial} and active learning \citep{settles2009active}.
In the real-world scenario, 
unimodal Gaussian distribution is widely used and offers computational advantages because its entropy can be calculated analytically, which offers computational advantages.
However, unimodal Gaussian distributions are limited in the approximation ability, and in particular, they hardly approximate multimodal distributions, which are often assumed, for instance, as posterior distributions of Bayesian neural networks (BNNs) \citep{mackay1992practical, neal2012bayesian} (see, e.g., \citet{fort2019deep}).
On the other hand, Gaussian mixture distribution (the superposition of Gaussian distributions) can capture the multimodality, and in general, approximate any continuous distribution in some topology (see, e.g., \citet{bacharoglou2010approximation}). 
Unfortunately, the entropy of Gaussian mixture lacks a closed form, 
thereby necessitating the use of entropy approximation methods.

There are numerous approximation methods for estimating the entropy of Gaussian mixture (see Section \ref{Related Works}).
One common method for approximating the entropy of a Gaussian mixture is to compute the weighted sum of the entropies of the individual unimodal Gaussian components with mixing coefficients (see \eqref{def:approx_of_entropy}), and this approximate entropy is easy to calculate analytically. However, despite its empirical success, this approximation lacks theoretical guarantees. The purpose of this paper is to reveal under what conditions this approximation performs well and to provide some theoretical insights on this approximation. Our contributions are as follows:

\begin{itemize}
\item 
{\bf (Main result: New error bounds)} We provide new upper and lower bounds of the approximation error. These bounds show that the error is essentially controlled by the ratios $\alpha_{k,k'}$ (which are given in Definition \ref{def:alpha}) of the distances between the means to the sum of the variances of each (Gaussian) component of the Gaussian mixture, and the error converges to zero as the ratios tend to infinity (Theorem~\ref{main_result_1}). 
Consequently, we provide an ``almost'' necessary and sufficient condition for the approximate entropy to be valid (Remark~\ref{rem:NScondition}).

\item {\bf (Probabilistic error bound)}
To confirm the effectiveness of the approximation for higher-dimension problems, we also provide the approximation error bound in the form of a probabilistic inequality (Corollary \ref{prob_ineq_cor}).
In supplementary of \citet{gal2016dropout}, it is mentioned (without mathematical proof) that this approximate entropy tends to be the true one in high-dimensional spaces when the means of the Gaussian mixture are randomly distributed. Our probabilistic inequality is a rigorous and generalized result of what they mention and shows the usefulness of this approximation in high-dimensional problems.
Moreover, we numerically demonstrate this result in a simple case and show its superiority over several other methods (Section \ref{sec:experiment-error}).

\item {\bf (Error bound for derivatives)} 
For example in machine learning, not only the approximation of entropy but also its partial derivatives are required in backpropagation. 
Therefore, we also provide the upper bounds for the partial derivatives of the error with respect to parameters (Theorem \ref{derivative-entropy}), which ensure that the derivatives of the approximate entropy are also close enough to the true ones when the ratios are large enough.

\item {\bf (More detailed analysis in the special case)}
We conduct a more detailed analysis of the error bounds in a special case. 
More precisely, when all covariance matrices coincide, we provide an explicit formula on the entropy of Gaussian mixture with the integral dimension-reduced to its component number (Proposition \ref{explicit form 1}). Then, by using this formula, we improve and simplify the upper and lower bounds of the error (Theorem \ref{common upper bound}) and the probabilistic inequality (Corollary~\ref{prob_ineq_common}). In this special case, we obtain a necessary and sufficient condition for this approximation to converge to the true entropy (Remark~\ref{rem:common}).
\end{itemize}

\section{Related work}\label{Related Works}

In numerical computation of the entropy of Gaussian mixtures, 
the approximation by Monte Carlo estimation is often used. 
However, it may require a large number of samples for high accuracy, leading to high computational costs. 
Furthermore, the Monte Carlo estimator gives a stochastic approximation, and hence, it does not guarantee deterministic bounds (confidence intervals may be used). 
There are numerous deterministic approximation methods (see, e.g., \citet{hershey2007approximating}). For example, approximation methods based on upper or lower bounds of the entropy are often used.
That is, we try to obtain an approximation by estimating the upper or lower bounds of the entropy or we adopt the bounds as an approximate entropy (see, e.g., \citet{bonilla2019generic,hershey2007approximating,nielsen2016guaranteed,zobay2014variational}). A typical one is to use the lower bound of the entropy based on Jensen’s inequality (see  \citet{bonilla2019generic}). 
These approximations have the advantage of being analytically calculated in closed forms, whereas there is no theoretical guarantee that they work well in the context of machine learning such as variational inference.
As another typical method, \citet{huber2008entropy} proposed the entropy approximation by a combination of the Taylor approximation with the splitting method of Gaussian mixture components. Recently, \citet{Dahlke2023Convergence} provided new approximations using Taylor and  Legendre series, and they 
theoretically and experimentally analyzed these approximations together with the approximation in \citet{huber2008entropy}. 
Notably, 
\citet{Dahlke2023Convergence} theoretically showed sufficient conditions for the approximations to be convergent. However, their performance remains unexplored in higher-dimensional cases.
As an approximation that is easy to calculate regardless of dimensions, the sum of the entropies of unimodal Gaussian distributions is often used.
For example, the approximation entropy represented as the ``sum with mixing coefficients'' of the entropies of unimodal Gaussian distributions (see Remark~\ref{rem:comp-melb}),
which is equal to the true one when all Gaussian components coincide, is investigated in \cite{melbourne2022differential}.
On the other hand, \citet{gal2016dropout} used the approximation represented as the sum of the entropies of ``unimodal Gaussian distributions with mixture coefficients'' in the context of variational inference from an intuition that this approximation tends to the true one in high-dimensional spaces when the means of the mixture are randomly distributed.
In our work, we focus on the theoretical error estimation for this approximation and reveal that this approximation converges to the true one when all Gaussian components are
far apart or in high-dimensional cases.

\section{Entropy approximation for Gaussian mixtures}\label{sec:Entropy-approximtion}
The entropy of probability distribution $q(w)$ is defined by
\begin{align}
\label{eq:entropy-definition}
H[q] \coloneqq -\int_{\mathbb{R}^m} q(w) \log(q(w))\,dw.
\end{align}
Here, we choose a probability distribution $q(w)$ as the Gaussian mixture distribution, that is,
\[
q(w)= \sum_{k=1}^{K}\pi_{k} \mathcal{N}(w\,|\,\mu_{k}, \Sigma_{k}), \quad w \in \mathbb{R}^m,
\]
where $K \in \mathbb{N}$ is the number of mixture components, and $\pi_k \in (0,1]$ are mixing coefficients constrained by $\sum_{k=1}^{K}\pi_k=1$.
Here, $\mathcal{N}(\mu_{k}, \Sigma_{k})=\mathcal{N}(w\,|\, \mu_{k}, \Sigma_{k})$ is the Gaussian distribution with a mean $\mu_k \in \mathbb{R}^{m}$ and (positive definite) covariance matrix $\Sigma_k \in \mathbb{R}^{m \times m}$, that is,
\[
\mathcal{N}(w\,|\,\mu_{k}, \Sigma_{k}) = \frac{1}{\sqrt{(2\pi)^{m} |\Sigma_{k}|}}\exp\left(-\frac{1}{2}\left\|w-\mu_k\right\|_{\Sigma_k}^{2}\right),
\]
where $|\Sigma_{k}|$ is the determinant of matrix $\Sigma_{k}$, and $\|x\|_{\Sigma}^{2} \coloneqq x \cdot( \Sigma^{-1}x)$ for a vector $x \in \mathbb{R}^{m}$ and a positive definite matrix $\Sigma \in \mathbb{R}^{m \times m}$.
However, the entropy term $H[q]$ cannot be computed analytically when the distribution $q(w)$ is a Gaussian mixture.

We define the approximate entropy $\widetilde{H}[q]$ by the sum of the entropies of ``unimodal Gaussian distributions with mixture coefficients'': 
\begin{equation}
\begin{split}
\widetilde{H}[q]
\coloneqq&\, -\sum_{k=1}^{K} \int_{\mathbb{R}^m} \pi_k\mathcal{N}(w\,|\,\mu_k, \Sigma_{k}) \log\left(\pi_{k} \mathcal{N}(w\,|\,\mu_k, \Sigma_{k}) \right)dw \\
= & \, \frac{m}{2} +  \frac{m}{2} \log 2 \pi + \frac{1}{2} \sum_{k=1}^{K} \pi_{k} \log |\Sigma_{k}| 
- \sum_{k=1}^{K} \pi_{k} \log \pi_{k}, 
\label{def:approx_of_entropy}
\end{split}
\end{equation}
which can be computed analytically.
It is obvious that 
$H[q]=\widetilde{H}[q]$ holds for unimodal Gaussian (i.e., the case $K=1$). Moreover, it is shown that 
\[
0 \le \widetilde{H}[q] - H[q] \le -2\sum_{k=1}^K \pi_k \log \pi_k\ (\leq 2\log K)
\]
(see \eqref{computation1} in Appendix and Remark \ref{rem:comp-melb}). 
These bounds show that the error does not blow up with respect to the mean $\mu_k$, the covariance $\Sigma_{k}$, and the dimension $m$ if the number of mixture components $K$ is fixed. In addition,
these bounds imply that $|\widetilde{H}[q] - H[q]| \to 0$ as 
the Gaussian mixture $q$ converges to a unimodal Gaussian (i.e., $\pi_k \to 1$ for some $k$ and $\pi_{k'} \to 0$ for all $k'\not =k$).
It is natural to ask under what other conditions this approximation $H[q] \approx \widetilde{H}[q]$ can be justified.
In Section \ref{entropy approximation}, we will provide an ``almost'' necessary and sufficient condition for this approximation to be valid.

\section{Theoretical error analysis of the entropy approximation}\label{entropy approximation}

In this section, we analyze the approximation error $|H[q] - \widetilde{H}[q]|$ to theoretically justify the entropy approximation $\widetilde{H}[q] \approx H[q]$.

We introduce the following notation to state our results.

\begin{dfn}\label{def:alpha}
For two Gaussian distributions $\mathcal{N}(\mu_{k}, \Sigma_{k})$ and $\mathcal{N}(\mu_{k^{\prime}}, \Sigma_{k^{\prime}})$, we define
$\alpha_{\{k,k'\}}$ by
\begin{equation} \label{alpha_(k)}
\alpha_{\{k,k'\}}\coloneqq \max
\Big\{
\alpha >0 : \{x \in \mathbb{R}^m : \|x-\mu_k\|_{\Sigma_k}< \alpha\} \cap \{x \in \mathbb{R}^m : \|x-\mu_{k'}\|_{\Sigma_{k'}}< \alpha\} = \varnothing
\Big\},
\end{equation}
and $\alpha_{k,k^{\prime}}$ by
\begin{equation}
\alpha_{k,k^{\prime}}\coloneqq\frac{\|\mu_{k}-\mu_{k^{\prime}}\|_{\Sigma_k} }{1+\|\Sigma_k^{-\frac{1}{2}}\Sigma_{k^{\prime}}^{\frac{1}{2}} \|_{\rm op}}, 
\label{alpha_k}
\end{equation}
where $k, k^{\prime} \in [K]\coloneqq\{1,\ldots,K\}$ and $\| \cdot \|_{\rm op}$ is the operator norm (i.e., the largest singular value).
\end{dfn}
\noindent
{\bf Interpretation of $\alpha$:}
We can interpret that $\alpha_{\{k,k'\}}$ and $\alpha_{k,k^{\prime}}$ measure distances of two Gaussian distributions in some sense respectively.
Here notice that $\alpha_{\{k,k'\}}$ is symmetric (i.e., $\alpha_{\{k,k'\}}$ is always equal to $\alpha_{\{k',k\}}$), but $\alpha_{k,k'}$ is not necessarily symmetric (i.e., $\alpha_{k,k'}$ is not always equal to $\alpha_{k',k}$).

Figure~\ref{illustration1} shows the geometric interpretation of $\alpha_{k,k^{\prime}}$ in the isotropic case, that is, $\Sigma_k=\sigma^{2}_{k}I$ and $\Sigma_{k^{\prime}}=\sigma^{2}_{k^{\prime}}I$.
In this case, 
$\alpha_{k,k^{\prime}}$ has a symmetric form with respect to $k, k'$ as
\[
\alpha_{\{k,k'\}} = \alpha_{k,k^{\prime}} = \alpha_{k^{\prime},k} = \frac{|\mu_{k}-\mu_{k^{\prime}}|}{\sigma_k + \sigma_{k^{\prime}}}.
\]
Here, the volume of $\mathcal{N}(\mu_k,\sigma_k^2 I)$ on $B(\mu_k,\alpha \sigma_k)$ is equal to that of $\mathcal{N}(\mu_{k'},\sigma_{k'}^2 I)$ on $B(\mu_{k'},\alpha \sigma_{k'})$,
where $B(\mu, \sigma)\coloneqq\{x \in \mathbb{R}^{m}: |x-\mu| < \sigma \}$ for $\mu \in \mathbb{R}^{m}$ and $\sigma>0$.
If all distances $|\mu_{k}-\mu_{k^{\prime}}|$ between the means go to infinity, or all variances $\sigma_k$ go to zero, etc., then $\alpha_{k, k^{\prime}}$ go to infinity for all pairs of $k, k^{\prime}$. 
Furthermore, if all means $\mu_k$ are normally distributed (variances $\sigma_k$ are fixed), then an expected value of $\alpha_{k,k^{\prime}}^{2}$ is in proportion to the dimension $m$.
That intuitively means that the expected value of $\alpha_{k, k^{\prime}}$ becomes large as the dimension $m$ of the parameter increases.

\begin{figure}[t]
\center
\includegraphics[keepaspectratio, scale=0.28]{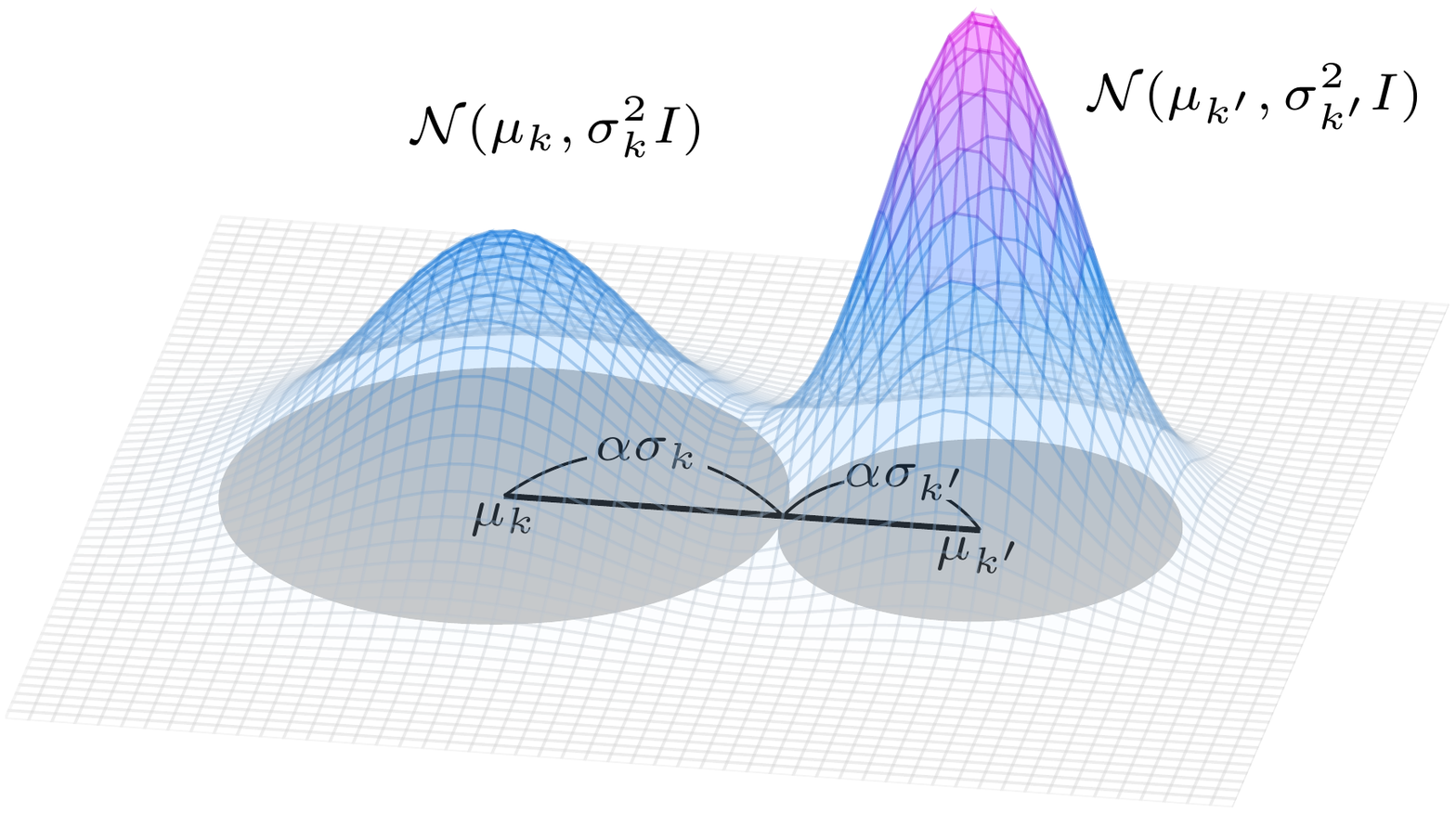}
\vspace{-1cm}
\caption{\small Illustration of $\alpha=\alpha_{\{k,k'\}}$ ($m=2$, isotropic)}
\label{illustration1}
\end{figure}
\begin{figure}[t]
\begin{minipage}[h]{0.45\linewidth}
\includegraphics[keepaspectratio, scale=0.25]{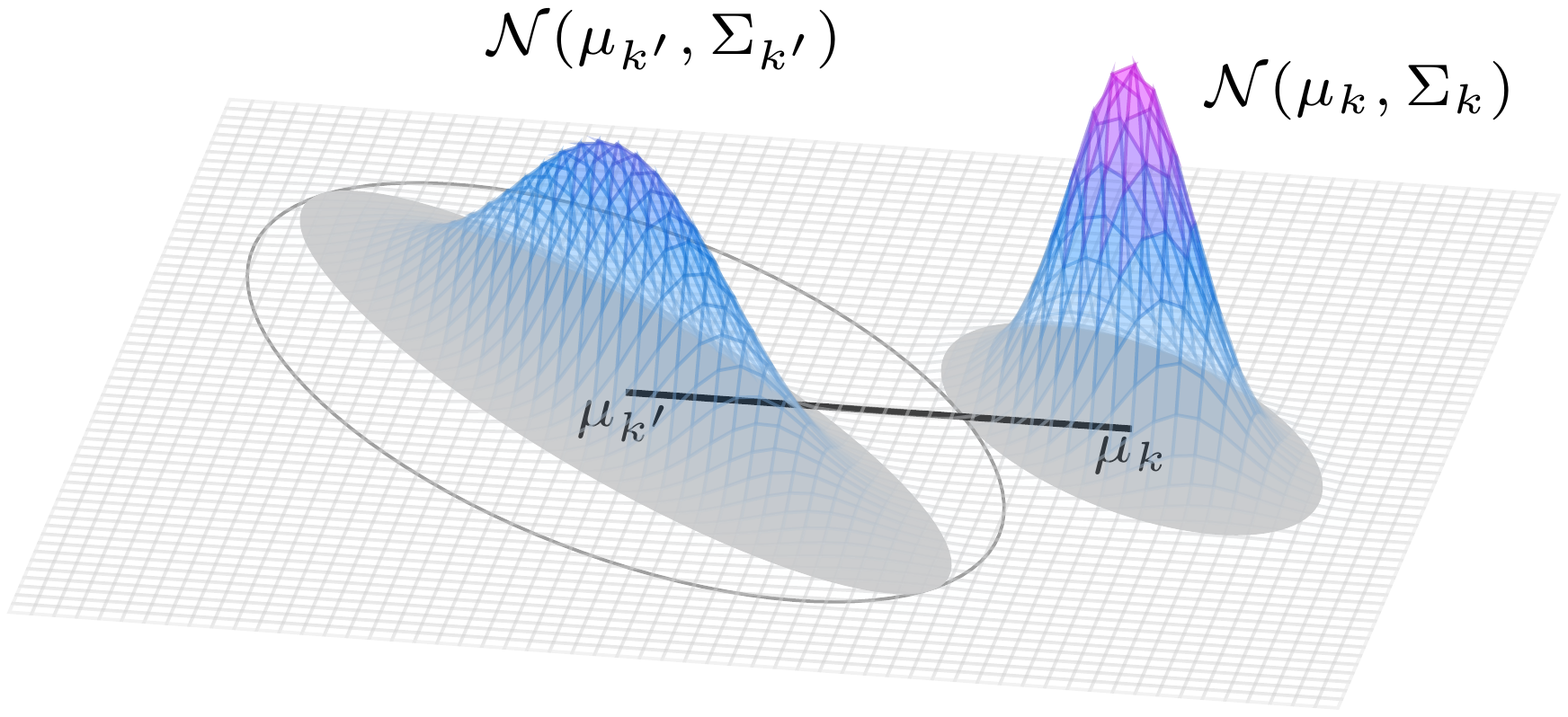}
\end{minipage}
\begin{minipage}[h]{0.45\linewidth}
\includegraphics[keepaspectratio, scale=0.28]{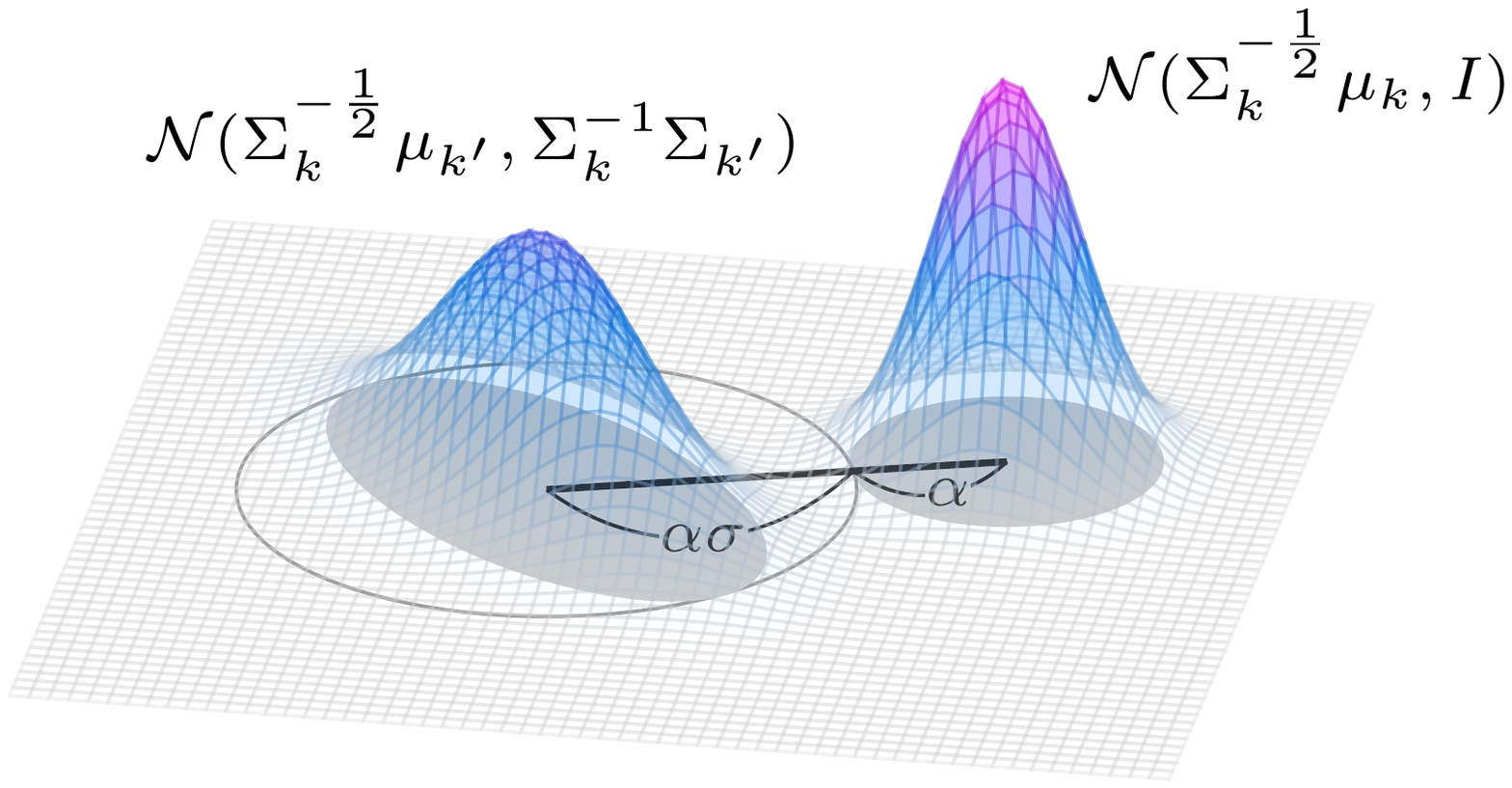}
\end{minipage}
\vspace{-1.5cm}
\caption{\small Illustration of $\alpha=\alpha_{k,k'}$ ($m=2$, anisotropic, $\sigma=\|\Sigma_k^{-\frac{1}{2}}\Sigma_{k'}^{\frac{1}{2}}\|_{\rm op})$}
\label{illustration2}
\vspace{20pt}
\end{figure}

These intuitions also hold in anisotropic case. From definition \eqref{alpha_k}, we have 
\begin{equation*}
\alpha_{k,k'}+\alpha_{k,k'}\|\Sigma_k^{-\frac{1}{2}}\Sigma_{k'}^{\frac{1}{2}}\|_{\rm op}=|\Sigma_k^{-\frac{1}{2}}\mu_k-\Sigma_k^{-\frac{1}{2}}\mu_{k'}|.
\end{equation*}
This shows that the circumsphere of $\{x \in \mathbb{R}^m : \|x-\Sigma_k^{-\frac{1}{2}}\mu_{k'}\|_{\Sigma_k^{-1}\Sigma_k'} < \alpha_{k,k'}\}$ circumscribes $B(\Sigma_k^{-\frac{1}{2}}\mu_k,\alpha)$ (see Figure~\ref{illustration2} right).
Moreover, by coordinate transformation $x \to \Sigma_k^{\frac{1}{2}}x$, we can interpret that $\alpha_{k,k'}$ is a distance of $\mathcal{N}(\mu_k,\Sigma_k)$ and $\mathcal{N}(\mu_{k'},\Sigma_{k'})$ from the perspective of $\Sigma_k$ (see Figure~\ref{illustration2} left), and $\alpha_{k,k'}$ gives a concrete form of $\alpha$ satisfying ${\{ \|x-\mu_k\|_{\Sigma_k}< \alpha\} \cap \{\|x-\mu_{k'}\|_{\Sigma_{k'}}< \alpha\} = \varnothing}$, that is, $\alpha_{k,k'} \leq \alpha_{\{k,k'\}}$ (Lemma~\ref{remark_alpha}).

\subsection{General covariance case}\label{General case}
We study the error $|H[q] - \widetilde{H}[q]|$ for general covariance matrices $\Sigma_k$.
First, we give the following upper and lower bounds for the error.
\begin{thm}\label{main_result_1}
Let $s \in [0,1)$. Then
\begin{equation}
\begin{split}
& 
\sum_{k=1}^{K}\sum_{k^{\prime}\neq k}
\frac{\pi_k\pi_{k^{\prime}}}{1-\pi_k}
c_{k,k^{\prime}}\log
\Biggl(
1+\frac{1-\pi_k}{\pi_k} \frac{|\Sigma_{k}|^{\frac{1}{2}}}{\displaystyle \max_{l}|\Sigma_{l}|^{\frac{1}{2}}}
\exp
\Biggl(
-\frac{
\Bigl(1+\|\Sigma_{k^{\prime}}^{-\frac{1}{2}}\Sigma_{k}^{\frac{1}{2}} \|_{\rm op}
\Bigr)^2
}{2}
\alpha_{k^{\prime},k}^{2}
\Biggr)
\Biggr) \\
& 
\leq 
\left|H[q] - \widetilde{H}[q] \right|\leq \frac{2}{(1-s)^{\frac{m}{4}}}\sum_{k=1}^{K}\sum_{k^{\prime}\neq k} \sqrt{\pi_k\pi_{k^{\prime}}} \exp\left(-\frac{s\alpha_{k,k^{\prime}}^{2}}{4} \right), \label{main_result_eq}
\end{split}
\end{equation}
where the coefficient $c_{k,k^{\prime}}$ is defined by
\begin{equation*}
c_{k,k^{\prime}}\coloneqq\frac{1}{\sqrt{(2\pi)^m}}\int_{\mathbb{R}^{m}_{k,k^{\prime}}}\exp\left(-\frac{|y|^2}{2}\right)dy \geq 0, \label{c_k}
\end{equation*}
and the set $\mathbb{R}^{m}_{k,k^{\prime}}$ is defined by
\[
\mathbb{R}^{m}_{k,k^{\prime}}
\coloneqq\left\{ y \in \mathbb{R}^m : 
\begin{array}{cc}
y\cdot y \geq (\Sigma_{k}^{\frac{1}{2}}\Sigma_{k^{\prime}}^{-1}\Sigma_{k}^{\frac{1}{2}}y)\cdot y,
\\
y\cdot (\Sigma_{k}^{\frac{1}{2}}\Sigma_{k^{\prime}}^{-1} (\mu_{k^{\prime}} -\mu_k) ) \geq 0 
\end{array}
\right\}.
\]
Moreover, the same upper bound holds for $\alpha_{\{k,k'\}}$ instead of $\alpha_{k,k'}$:
\begin{equation}\label{main_result_eq_2}
\left| H[q] - \widetilde{H}[q] \right|\leq \frac{2}{(1-s)^{\frac{m}{4}}}\sum_{k=1}^{K}\sum_{k^{\prime}\neq k} \sqrt{\pi_k\pi_{k^{\prime}}} \exp\left(-\frac{s\alpha_{\{k,k'\}}^{2}}{4} \right).
\end{equation}
\end{thm}

The proof is given by calculations with the triangle inequality, the Cauchy-Schwarz inequality, and a change of variables. 
For upper bound, we use the inequality $\log(1+ x) \leq \sqrt{x}$ $(x \geq 0)$, and  we split the integral region to $|y|<\alpha_{k,k'}$ and $|y|>\alpha_{k,k'}$ in order to utilize the characteristic of $\alpha_{k,k'}$. 
While, for the lower bound, we use the concavity of function $x \mapsto \log(1+x)$ $(x > 0)$ and the definition of $\mathbb{R}^m_{k,k'}$.
See Appendix~\ref{Proof of Theorem 4.2} for the proof of Theorem \ref{main_result_1}.

\begin{rem}\label{rem:NScondition}
Theorem \ref{main_result_1} implies the following facts:
\begin{itemize}
\item[\rm (i)] 
According to the upper bound in \eqref{main_result_eq}, the approximation $H[q] \approx \widetilde{H}[q]$ is valid when 
\begin{itemize}
    \item[\rm (a)] $\pi_k \pi_{k'} \to 0$ or $\alpha_{k,k^{\prime}} \to +\infty$ for all pairs $k, k^{\prime} \in [K]$ with $k \neq k^{\prime}$
\end{itemize}
for $s\in (0,1)$.
In particular, 
the error exponentially decays to zero 
when all $\alpha_{k,k'}$ go to zero. 
\item[\rm (ii)]
According to the upper bound \eqref{main_result_eq_2}, the approximation $H[q] \approx \widetilde{H}[q]$ is also valid when 
\begin{itemize}
    \item[\rm (b)] $\pi_k \pi_{k'} \to 0$ or $\alpha_{\{k,k'\}} \to +\infty$ for all pairs $k, k^{\prime} \in [K]$ with $k \neq k^{\prime}$
\end{itemize}
for $s\in (0,1)$.
Moreover, the upper bound in \eqref{main_result_eq_2} is better than that in \eqref{main_result_eq} since $\alpha_{k,k'} \leq \alpha_{\{k,k'\}}$ always holds (see Lemma~\ref{remark_alpha}).
\item[\rm (iii)]
According to the lower bound in \eqref{main_result_eq}, if $c_{k,k^{\prime}}>0$ and $\|\Sigma_{k'}^{-\frac{1}{2}}\Sigma_k^{\frac{1}{2}}\|_{\rm op}\le C$ hold for all pairs $k,k'$ and some constant $C>0$ independent of $k,k'$, then (a) is necessary for this approximation to be valid (where we note that $|\Sigma_k|^{\frac{1}{2}}/|\Sigma_{k'}|^{\frac{1}{2}}=|\Sigma_{k'}^{-\frac{1}{2}}\Sigma_k^{\frac{1}{2}}| \le C^m$ when $\|\Sigma_{k'}^{-\frac{1}{2}}\Sigma_k^{\frac{1}{2}}\|_{\rm op}\le C$, and that $\min_k|\Sigma_k|^{\frac{1}{2}}/\max_k|\Sigma_k|^{\frac{1}{2}}\geq C^{-m}$ holds).
It is unclear whether all $c_{k,k^{\prime}}$ are positive, but we can show that either $c_{k,k'}$ or $c_{k',k}$ is always positive for any pair $k,k'$ (see Remark \ref{remarkA}).
\item[\rm (iv)]
From the above facts and the symmetry of $\alpha_{\{k,k'\}}$, we conclude that (b) is a necessary and sufficient condition for the approximation $H[q] \approx \widetilde{H}[q]$ to be valid if $\|\Sigma_{k'}^{-\frac{1}{2}}\Sigma_k^{\frac{1}{2}}\|_{\rm op}\le C$ hold for all pairs $k,k'$ and some constant $C>0$ independent of $k, k'$.
\end{itemize}
\end{rem}

\begin{rem}
In Theorem~\ref{main_result_1}, 
the parameter $s$ plays a role of adjusting the convergence speed as follows:
\begin{enumerate}
\renewcommand{\labelenumi}{\textup{(\roman{enumi})}}
    \item In the case of $s=0$, the upper bound does not imply the convergence, and the following bound is obtained:
    \begin{equation*}
        \left|H[q]-\widetilde{H}[q]\right| \leq 2\sum_{k=1}^{K}\sum_{k^{\prime}\neq k} \sqrt{\pi_k\pi_{k^{\prime}}}
        \leq \sum_{k=1}^{K}\sum_{k^{\prime}\neq k} (\pi_k+\pi_k') 
        = \sum_{k=1}^K \{(K-1)\pi_k+(1-\pi_k) \}
        =2(K-1).
    \end{equation*}
    \item Consider the upper bound (focusing on the term of $k,k'$ pair in summation) of~\eqref{main_result_eq} as a function with respect to $s$, it is minimal on $s=1-m/\alpha_{k,k'}^2$ if $\alpha_{k,k'} \geq \sqrt{m}$, and monotonically increase in $s \in [0,1)$ if $\alpha_{k,k'}<\sqrt{m}$.
    Replacing $s$ on each $k,k'$ summation in the proof of Theorem~\ref{main_result_1} with minimal points $s_{k,k'}=1-m/\alpha_{k,k'}^2$, we obtain more precise upper bound
    \begin{equation*}
        \left|H[q] - \widetilde{H}[q] \right|
        \leq 2 \sum_{k=1}^K \sum_{k' \neq k} \sqrt{\pi_k \pi_{k'}}\left(\frac{\alpha_{k,k'}^2}{m}\right)^\frac{m}{4}\exp\left(\frac{m-\alpha_{k,k'}^2}{4}\right),
    \end{equation*}
    which converges to zero when dimension $m$ goes to infinity, if $\alpha_{k,k'}>\sqrt{m}$ for all pairs $k,k'$. 
\end{enumerate}
\end{rem}

\begin{rem}
\label{rem:comp-melb}
\citet{melbourne2022differential} has explored the entropy approximation of mixtures represented as the ``sum with mixing coefficients'' of the entropies of unimodal Gaussian distributions
(without mixing coefficients in the logarithmic term):
\begin{equation}\label{H_melbourne}
    \widetilde{H}_{\rm Melbourne}[q]
    \coloneqq -\sum_{k=1}^{K}\pi_{k} \int_{\mathbb{R}^m} \mathcal{N}(w\,|\,\mu_k, \Sigma_{k}) \log\left(\mathcal{N}(w\,|\,\mu_k, \Sigma_{k}) \right)dw
    =\widetilde{H}[q]+\sum_{k=1}^K \pi_k \log \pi_k.
\end{equation}
This approximate entropy is equal to the true entropy $H[q]$ when all Gaussian components coincide (i.e., $\mu_k=\mu_{k'}$ and $\Sigma_k=\Sigma_{k'}$ for all $k,k' \in [K])($see \cite[Theorem I.1]{melbourne2022differential}), but not when all Gaussian components are far apart (i.e., $\alpha_{k,k'}\to \infty$ for all $k,k' \in[K]$).
In contrast, while the approximation $\widetilde{H}[q]$ differs from $H[q]$ when all Gaussian components coincide, it converges to $H[q]$ when all Gaussian components are far apart (see Theorem~\ref{main_result_1}).
This is because $\widetilde{H}
_{\rm Melbourne}[q]$ differs by $-\sum_{k=1}^{K} \pi_{k} \log \pi_{k}$ from $\widetilde{H}[q]$ (refer to \eqref{def:approx_of_entropy}).
Moreover, using also \citet[Lemma XI.2]{wang2014beyond}, the other upper bound
\begin{equation}\label{bound:log}
    \left|H[q] - \widetilde{H}[q] \right|
    \leq \Bigl|H[q]-H_{\rm Melbourne}[q]\Bigr|+\left|H_{\rm Melbourne}[q]-\widetilde{H}[q]\right|
    \leq -2\sum_{k=1}^K \pi_k \log \pi_k \leq 2\log K
\end{equation}
is obtained, where the last inequality is justified by constraint $\sum_{k=1}^K \pi_k =1$.
\end{rem}

\begin{rem}
In Lemma~\ref{part1}, we obtain another upper bound $K/2$, which is slightly better than 
$2\log K$ in \eqref{bound:log} when $2\le K\le 8$. 
\end{rem}


Next, we provide the probabilistic inequality for the error as a corollary of Theorem~\ref{main_result_1}.

\begin{cor}\label{prob_ineq_cor}
Let $c>0$.
Take 
$\{\mu_k\}_k$ and $\{\Sigma_{k}\}_k$ such that 
\begin{equation}
\frac{\Sigma_{k}^{-\frac{1}{2}}(\mu_k-\mu_{k^{\prime}})}{1+\|\Sigma_k^{-\frac{1}{2}}\Sigma_{k^{\prime}}^{\frac{1}{2}} \|_{\rm op}} \sim \mathcal{N}(0, c^{2}I)
\label{random_mu_k}
\end{equation}
for all pairs $k, k^{\prime} \in [K]$ ($k\neq k^{\prime}$), 
that is, the left-hand side follows a Gaussian distribution with zero mean and an isotropic covariance matrix $c^{2}I$.
Then, for $\varepsilon>0$ and $s \in (0,1)$,
\begin{equation}
\begin{split}
&P\left(\left|H[q] - \widetilde{H}[q] \right| \geq \varepsilon \right) \leq \frac{2(K-1)}{\varepsilon}\left(\sqrt{1-s} \left(1+\frac{s c^2}{2}\right) \right)^{-\frac{m}{2}}.
\end{split}
\label{prob_ineq_gen}
\end{equation}
\end{cor}
The proof is a combination of Markov's inequality and the upper bound in Theorem~\ref{main_result_1}. 
We also use the moment generating function of $\alpha_{k,k^{\prime}}^{2}/c^2$ which follows the $\chi^{2}$-distribution  by the assumption \eqref{random_mu_k}.
See Appendix~\ref{Proof of Corollary4.3} for the details.

When \eqref{random_mu_k} holds, an expected value of $\alpha^2_{k,k'}$ is $c^2 m$. Hence, if $\alpha^2_{k,k'}$ is regarded as $c^2 m$, then 
for $c>1$ there exists $s\in (0,1)$ such that the upper bound in \eqref{main_result_eq} expectedly converges to zero as the dimension $m$ goes to infinity.
Furthermore, Corollary \ref{prob_ineq_cor} justifies \citet[Proposition 1 in Appendix A]{gal2016dropout}, which formally mentions that $\widetilde{H}[q]$ tends to $H[q]$ when means $\mu_k$ are normally distributed, all elements of covariance matrices $\Sigma_k$ do not depend on $m$, and $m$ is large enough.
In fact, the right-hand side of \eqref{prob_ineq_gen} converges to zero as $m \to \infty$ for some $s \in (0,1)$ if $c > 1$.
\vspace{3mm}
\par
We also study the derivatives of the error $|H[q] - \widetilde{H}[q]|$ with respect to learning parameters $\theta= (\pi_{k}, \mu_{k}, \Sigma_{k})_{k=1}^{K}$.
For simplicity, we write
\[
\Gamma_k\coloneqq\Sigma^{\frac{1}{2}}_k.
\]
We give the following upper bounds for the derivatives of the error.
\begin{thm}\label{derivative-entropy}
Let $k \in [K]$, $p,q \in [m]$, and $s \in (0,1)$. Then
\begin{align*}
\mathrm{(i)} 
&
\quad
\left| \frac{\partial}{\partial \mu_{k,p}} \left(H[q] - \widetilde{H}[q]\right) \right|
\leq
\frac{2}{(1-s)^{\frac{m+2}{4}}}
\sum_{k^{\prime}\neq k}
\sqrt{\pi_k \pi_{k^{\prime}}}
\left( \left\| \Gamma_{k^{\prime}}^{-1} \right\|_{1}+
\left\| \Gamma_{k}^{-1} \right\|_{1} \right)
\exp
\left(
-\frac{s\alpha_{k,k'}^{2}}{4}
\right),
\\
\hspace{-0.5cm}
\mathrm{(ii)} 
&
\quad
\left| \frac{\partial}{\partial \gamma_{k,pq}} \left(H[q] - \widetilde{H}[q]\right) \right| 
\notag \\
& \quad \leq
\frac{6}{(1-s)^{\frac{m+4}{4}}}
\sum_{k^{\prime}\neq k}
\sqrt{\pi_k \pi_{k^{\prime}}}
\left(
2|\Gamma_k|^{-1}|\Gamma_{k, pq}|
+
\left\| \Gamma_{k}^{-1} \right\|_{1}
+
\left\| \Gamma_{k'}^{-1} \right\|_{1}
\right)
\exp
\left(
-\frac{s\alpha_{k,k'}^{2}}{4}
\right)
\\ &\quad \text{for} \ \gamma_{k,pq}\in \mathbb{R} \ \text{satisfying} \ \|\Gamma_k^{-1}\|_1<\infty,\\
\mathrm{(iii)} 
&
\quad
\left| \frac{\partial}{\partial \pi_{k}} \left(H[q] - \widetilde{H}[q]\right) \right|
\leq
\frac{8}{(1-s)^{\frac{m}{4}}}
\sum_{k^{\prime}\neq k}
\sqrt{\frac{\pi_{k^{\prime}}}{\pi_k}}
\exp
\left(
-\frac{s\alpha_{k,k'}^{2}}{4}
\right),
\end{align*}
where $\mu_{k,p}$ and $\gamma_{k,pq}$ is the $p$-th and $(p,q)$-th components of vector $\mu_k$ and matrix $\Gamma_{k}$, respectively, and $\left\| \cdot \right\|_{1}$ is the entry-wise matrix $1$-norm, and $|\Gamma_{k,pq}|$ is the determinant of the $(m-1)\times(m-1)$ matrix that results from deleting $p$-th row and $q$-th column of matrix $\Gamma_{k}$.
Moreover, the same upper bounds hold for $\alpha_{\{k,k'\}}$ instead of $\alpha_{k,k'}$.
\end{thm}

The proof is given by similar calculations and techniques to the proof of Theorem~\ref{main_result_1}.
For the details, see Appendix~\ref{sec-proof of derivative entropy}. 

We observe that even in the derivatives of the error, the upper bounds exponentially decay to zero as $\alpha_{k,k^{\prime}}$ go to infinity for all pairs $k, k^{\prime} \in [K]$ with $k \neq k^{\prime}$. 
We can also show that if means $\mu_k$ are normally distributed with certain large standard deviation $c$, then the probabilistic inequality like Corollary~\ref{prob_ineq_cor} that the bound converges to zero as $m$ goes to infinity is obtained.

\subsection{Coincident covariance case}\label{sec:coincident-covariance-case}
We study the error $|H[q] - \widetilde{H}[q]|$ for coincident covariance matrices, that is,
\begin{equation}
\label{assumption_common}
\Sigma_k = \Sigma\quad \mbox{for all} \ k \in [K],
\end{equation}
where $\Sigma \in \mathbb{R}^{m \times m}$ is a positive definite matrix.
In this case, $\alpha_{k,k^{\prime}}$ have the form as
\[
\alpha_{\{k,k'\}}=\alpha_{k,k^{\prime}}=\alpha_{k',k}=\frac{\left\|\mu_{k}-\mu_{k^{\prime}}\right\|_{\Sigma} }{2}.
\]
In this case, a more detailed analysis can be done.
First, we show the following explicit form of the true entropy $H[q]$.
\begin{prop}\label{explicit form 1}
Let $m \geq K \geq 2$.
Then
\begin{equation}
H[q]  = \widetilde{H}[q] - \sum_{k=1}^{K}\frac{\pi_k}{(2\pi)^{\frac{K-1}{2}}}
\int_{\mathbb{R}^{K-1}} \exp\left(-\frac{|v|^2}{2}\right)
\log\left(1 + \sum_{k^{\prime}\neq k}  \frac{\pi_{k^{\prime}}}{\pi_k}\exp\left(\frac{|v|^2 - \left|v-u_{k^{\prime},k}\right|^{2}}{2}\right) \right)dv,
\label{explicit_form}
\end{equation}
where $u_{k^{\prime},k}\coloneqq[R_{k} \Sigma^{-1/2} (\mu_{k^{\prime}} - \mu_{k})]_{1:K-1} \in \mathbb{R}^{K-1}$ and $R_{k} \in \mathbb{R}^{m \times m}$ is some rotation matrix such that
\begin{equation}
\label{assumption_rotation}
R_{k}\Sigma^{-\frac{1}{2}}(\mu_{k^{\prime}} - \mu_k ) \in \mathrm{span}\{e_1, \cdots,e_{K-1}\}, \,\,\, k^{\prime} \in [K].
\end{equation}
Here, $\{e_i\}_{i=1}^{K-1}$ is the standard basis in $\mathbb{R}^{K-1}$, and
$u_{1:K-1}\coloneqq(u_1,\ldots,u_{K-1})^{T}\in \mathbb{R}^{K-1}$ for $u=(u_1,\ldots,u_m)^{T}\in \mathbb{R}^{m}$.
\end{prop}
The proof is given by certain rotations and polar transformations. For the details, see Appendix \ref{Proof of Lemma 4.4}.

We note that the special case $K=2$ of \eqref{explicit_form} can be found in \citet[Appendix A]{zobay2014variational}.
Using Proposition \ref{explicit form 1}, we have the following upper and lower bounds for the error.
\begin{thm}\label{common upper bound}
Let $m \geq K \geq 2$ and $s \in [0,1)$.
Then
\begin{align}
&\frac{1}{2}\sum_{k=1}^{K}\sum_{k^{\prime}\neq k} \frac{\pi_k\pi_{k^{\prime}}}{1-\pi_k}
\log\left(1+ \frac{1-\pi_k}{\pi_k}\exp(-2\alpha_{k^{\prime},k}^{2})\right) \\
&\leq \left|H[q] - \widetilde{H}[q] \right|
\leq
\frac{2}{(1-s)^{\frac{K-1}{4}}}\sum_{k=1}^{K}\sum_{k^{\prime}\neq k} \sqrt{\pi_k\pi_{k^{\prime}}} \exp\left(-\frac{s\alpha_{k,k^{\prime}}^{2}}{4}\right). \label{common upper bound_upper}
\end{align}
\end{thm}
The upper bound is proved by the argument in proof of the second upper bound in Theorem~\ref{main_result_1} with the explicit form \eqref{explicit_form} and $|u_{k,k^{\prime}}|= \left\|\mu_k -\mu_{k^{\prime}} \right\|_{\Sigma}$. 
The lower bound is given by applying that in Theorem~\ref{main_result_1} in the case of $\Sigma_k=\Sigma$ for all $k \in [K]$ by remarking $c_{k,k^{\prime}}=1/2$ in that case.

\begin{rem} \label{rem:common}
Theorem \ref{common upper bound} implies the following facts:
\begin{itemize}
    \item[\rm (i)] When all covariance matrices coincide, the condition (a) or (b) in Remark~\ref{rem:NScondition} is a necessary and sufficient condition for the approximation $H[q] \approx \widetilde{H}[q]$ to be valid.
    \item[\rm (ii)] The upper bound in \eqref{common upper bound_upper} is sharper than that in \eqref{main_result_eq} of Theorem \ref{main_result_1} when $m \geq K$.
\end{itemize}
\end{rem}

\begin{cor}\label{prob_ineq_common}
Let $m \geq K \geq 2$ and $c>0$.
Take 
$\{\mu_k\}_k$ and $\Sigma$ such that 
\[
\frac{\Sigma^{-\frac{1}{2}}(\mu_k-\mu_{k^{\prime}})}{2} \sim \mathcal{N}(0, c^{2}I),
\]
for all pairs $k,k^{\prime} \in [K]$ ($k\neq k^{\prime}$). 
Then, for $\varepsilon>0$ and $s \in (0,1)$,
\[
P\left(\left|H[q] - \widetilde{H}[q] \right| \geq \varepsilon \right) \leq \frac{2(K-1)}{\varepsilon (1-s)^{\frac{K-1}{4}}}\left(1+\frac{s c^2}{2} \right)^{-\frac{m}{2}}.
\]
\end{cor}
The proof is given by the same way as in the proof of Corollary \ref{prob_ineq_cor}, which uses 
the Markov's inequality, the upper bound of Theorem~\ref{common upper bound}, and the moment generating function for $\chi^{2}$-distribution.

Note that, in Corollary~\ref{prob_ineq_common}, the assumption $c>1$, which is required in Corollary~\ref{prob_ineq_cor}, is no longer necessary for zero convergence.
\section{Experiment}\label{sec:experiment-error}

We numerically examined the approximation capabilities of the approximate entropy \eqref{def:approx_of_entropy} compared with \citet{huber2008entropy}, \citet{bonilla2019generic}, and the Monte Carlo integration.
Generally, we cannot compute the entropy~\eqref{eq:entropy-definition} in a closed form.
Therefore, we restricted the setting of the experiment to the case for the coincident covariance matrices (Section~\ref{sec:coincident-covariance-case}), in particular $\Sigma = I$, and 
the number of mixture components $K=2$, where we obtained the more tractable formula for the entropy~\eqref{eq:entropy-k2}.
In this setting, we investigated the relative error between the entropy and each approximation method (see Figure~\ref{fig:error-graph}).
The details for the experimental setting and exact formulas for each method are shown in Appendix~\ref{app:experiment-error}.

From the result in Figure~\ref{fig:error-graph}, we can observe the following.
First, the relative error of ours shows faster decay than others in higher dimensions $m$.
Therefore, the approximate entropy \eqref{def:approx_of_entropy} has an advantage in higher dimensions.
Second, the graph for ours scales in the $x$-axis as $c$ scales, which is consistent with the expression of the upper bound in Corollary~\ref{prob_ineq_common}.
Finally, ours is robust against varying mixing coefficients, which cannot be explained by Corollary~\ref{prob_ineq_common}.
Note that we can hardly conduct a similar experiment for $K>2$ because we cannot prepare the tolerant ground truth of the entropy.
For example, even the Monte Carlo integration is not suitable for the ground truth already in the case for $K=2$ due to its large relative error around $10^{-3}$.

\begin{figure*}[t]
\centering
\includegraphics[width=0.7\linewidth]{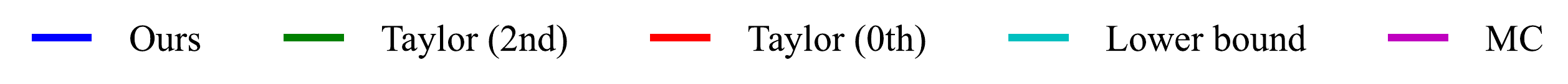} \\
\begin{minipage}[b]{0.375\linewidth}
    \centering
    \includegraphics[width=\linewidth]{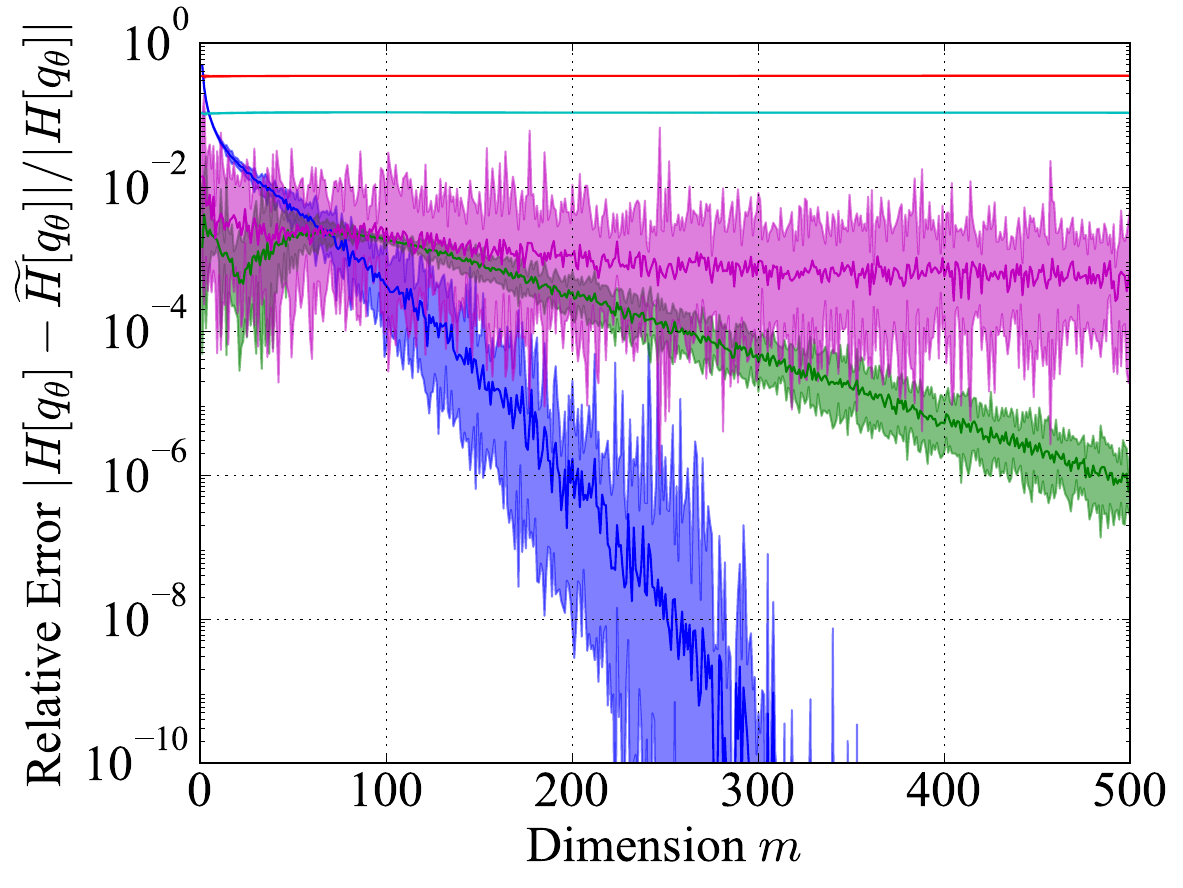}
    \begin{center}\vspace{-8pt}\small (a) \ $c=0.1$, $\pi_1=0.5$, $\pi_2=0.5$\end{center}
\end{minipage}
\hspace{1cm}
\begin{minipage}[b]{0.375\linewidth}
    \centering
    \includegraphics[width=\linewidth]{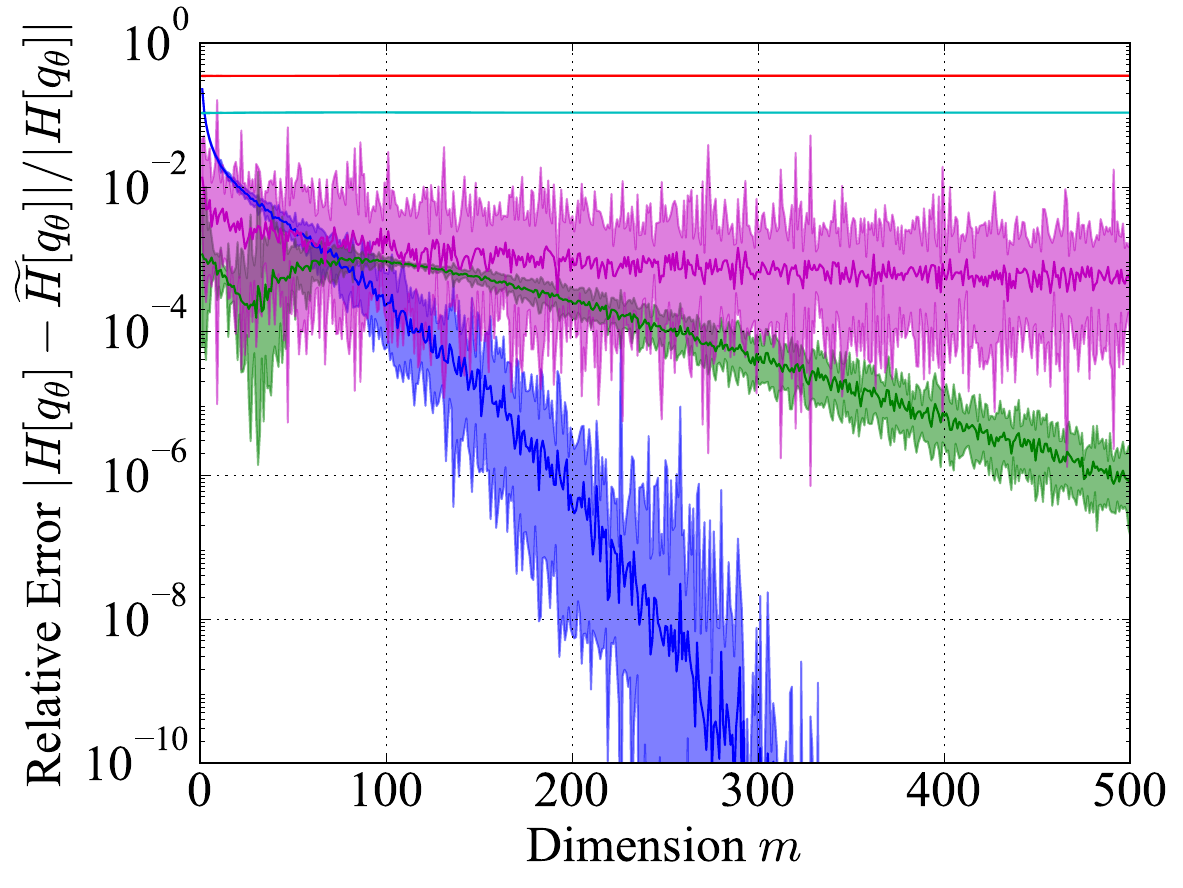}
     \begin{center}\vspace{-8pt}\small (b) \ $c=0.1$, $\pi_1=0.1$, $\pi_2=0.9$\end{center}
\end{minipage}\vspace{5pt}
\\
\begin{minipage}[b]{0.375\linewidth}
    \centering
    \includegraphics[width=\linewidth]{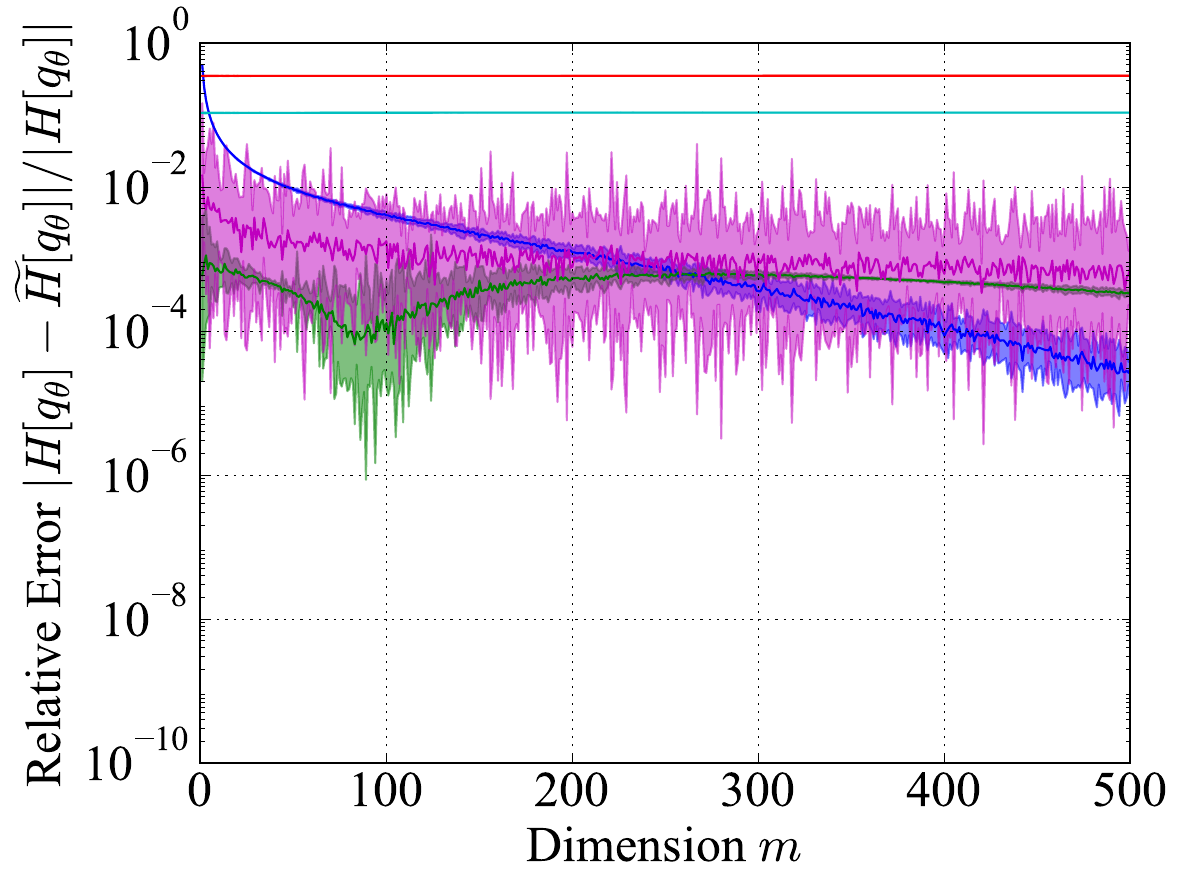}
     \begin{center}\vspace{-8pt}\small (c) $c=0.05$, $\pi_1=0.5$, $\pi_2=0.5$\end{center}
\end{minipage}
\hspace{1cm}
\begin{minipage}[b]{0.375\linewidth}
    \centering
    \includegraphics[width=\linewidth]{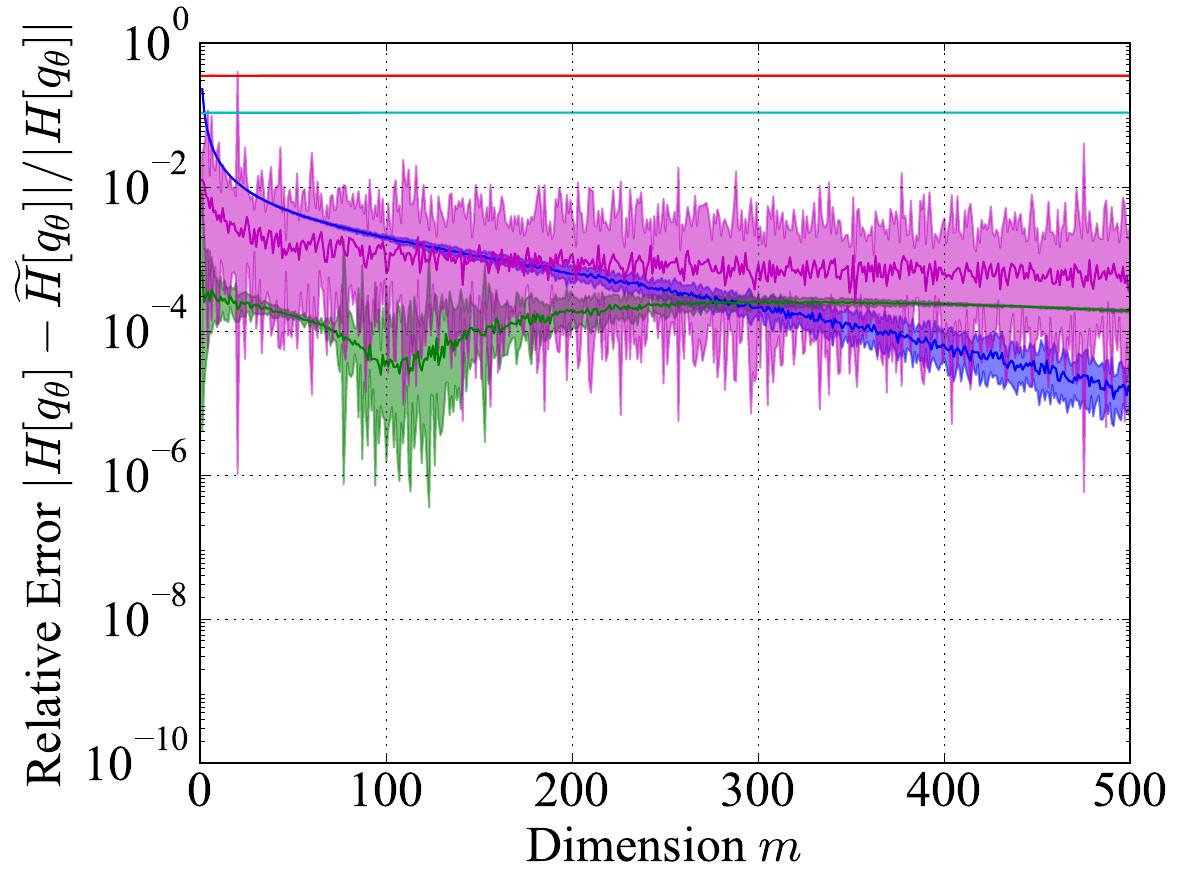}\begin{center}\vspace{-8pt}\small (d) $c=0.05$, $\pi_1=0.1$, $\pi_2=0.9$\end{center}
\end{minipage}\vspace{5pt}
\caption{\small Relative error $|H[q] - \widetilde{H}_*[q]|/|H[q]|$ for the true entropy $H[q]$ and the approximation ones $\widetilde{H}_*[q]$.
Each line indicates the mean value for 500 samples and the filled region indicates the min--max interval.
$c$ denotes the same symbol in Corollary~\ref{prob_ineq_common}.
Methods of Ours, Taylor ($2$nd), Taylor ($0$th), and Lower bound denote $\widetilde{H}_{{\rm ours}}[q]$, $\widetilde{H}_{{\rm Huber}(2)}[q]$, $\widetilde{H}_{{\rm Huber}(0)}[q]$, and $\widetilde{H}_{{\rm Bonilla}}[q]$ in Appendix~\ref{app:experiment-error}, respectively.
A method of MC denotes the Monte Carlo integration with $1000$ sampling points.
}
\label{fig:error-graph}
\vspace{20pt}
\end{figure*}

\section{Limitations and future work}

The limitations and future work are as follows:

\begin{itemize}

\item When all covariance matrices coincide, a necessary and sufficient condition is obtained for the approximation $H[q] \approx \tilde{H}[q]$ to be valid (see (i) in Remark~\ref{rem:common}). However, in the general covariance case, it has not yet been obtained without the constraints for covariance matrices (Remark~\ref{rem:NScondition}).
Improving the lower bound (or upper bound) to find a necessary and sufficient condition is a future work.

\item
There is an unsolved problem on the standard deviation $c$ in \eqref{random_mu_k} of Corollary \ref{prob_ineq_cor}. 
According to this corollary,
the approximation error almost surely converges to zero as $m\to\infty$
if we take $c>1$ (the discussion after Corollary \ref{prob_ineq_cor}). However, it is unsolved whether the condition $c>1$ is 
optimal or not for the convergence. 
According to Corollary \ref{prob_ineq_common}, the condition $c>1$ can be removed in the particular case
$\Sigma_k = \Sigma$ for all $k \in [K]$. 

\item The approximate entropy \eqref{def:approx_of_entropy} is valid only when $\alpha_{k,k^{\prime}}$ are large enough.
However, since there are situations where $\alpha_{k,k^{\prime}}$ are likely to be small, such as the low-dimensional latent space of a variational autoencoder \citep{kingma2013auto}, it is worthwhile to propose an appropriate entropy approximation for small $\alpha_{k,k^{\prime}}$ situation.
Although approximation $H_{\rm Melbourne}$ of~\eqref{H_melbourne} seems to be an appropriate one for small $\alpha_{k,k'}$ situation, the criteria (e.g., some value of $\alpha_{k,k'}$) for using either of $\widetilde{H}$ and $H_{\rm Melbourne}$ is unclear.

\item 
Further enrichment of experiments is important, such as the large component case, and comparison of the derivatives of entropy.

\item 
One important application of the entropy approximation is variational inference. In Appendix~\ref{Variational inference for Bayesian neural networks}, we include an overview of variational inference and an experiment on the toy task. 
However, these are not sufficient to determine the effectiveness of this approximation for variational inference. 
For instance, the variational inference maximizes the ELBO in \eqref{eq:elbo}, which includes the entropy term.
Investigating how this approximate entropy dominates other terms will be interesting for future work.
\end{itemize}

\subsection*{Acknowledgments}
TF was partially supported by JSPS KAKENHI Grant Number JP24K16949.
\newpage
\bibliographystyle{plainnat}
\bibliography{arxiv_Info-Infer.bbl}

\newpage

\appendix
\part*{Appendix}
\section{Proofs in Section~\ref{entropy approximation}}

We recall the definitions and notations used in this appendix. 
Let $q(w)$ be the Gaussian mixture distribution, that is,
\[
q(w)= \sum_{k=1}^{K}\pi_{k} \mathcal{N}(w\,|\,\mu_{k}, \Sigma_{k}), \quad w \in \mathbb{R}^m,
\]
where $m\in \mathbb N$ is the dimension, $K \in \mathbb{N}$ is the number of mixture components, $\pi_k \in (0,1]$ are mixing coefficients constrained by $\sum_{k=1}^{K}\pi_k=1$, and 
$\mathcal{N}(w\,|\, \mu_{k}, \Sigma_{k})$ is the Gaussian distribution with a mean $\mu_k \in \mathbb{R}^{m}$ and covariance matrix $\Sigma_k \in \mathbb{R}^{m \times m}$, that is,
\[
\mathcal{N}(w\,|\,\mu_{k}, \Sigma_{k}) = \frac{1}{\sqrt{(2\pi)^{m} |\Sigma_{k}|}}\exp\left(-\frac{1}{2}\left\|w-\mu_k\right\|_{\Sigma_k}^{2}\right).
\]
Here, $|\Sigma_{k}|$ is the determinant of matrix $\Sigma_{k}$, and $\|x\|_{\Sigma}^{2}\coloneqq x \cdot( \Sigma^{-1}x)$ for a vector $x \in \mathbb{R}^{m}$ and a positive definite matrix $\Sigma \in \mathbb{R}^{m \times m}$.
The entropy of $q(w)$ and its approximation are defined by 
\begin{align*}
H[q] \coloneqq& -\int q(w) \log( q(w))\,dw,\\
\widetilde{H}[q]
\coloneqq&\, -\sum_{k=1}^{K}\pi_{k} \int \mathcal{N}(w\,|\,\mu_k, \Sigma_{k}) \log\left(\pi_{k} \mathcal{N}(w\,|\,\mu_k, \Sigma_{k}) \right)dw \\
= & \, \frac{m}{2} +  \frac{m}{2} \log 2 \pi + \frac{1}{2} \sum_{k=1}^{K} \pi_{k} \log |\Sigma_{k}| 
- \sum_{k=1}^{K} \pi_{k} \log \pi_{k}. 
\end{align*}

\subsection{Proof of Theorem~\ref{main_result_1} }\label{Proof of Theorem 4.2}

Theorem~\ref{main_result_1} is a combination of Lemmas~\ref{part2}, \ref{part2.1}, and \ref{part2.5} stated below.

\begin{lem}\label{part1}
\[
\left|H[q] - \widetilde{H}[q] \right| \leq \frac{K}{2}.
\]
\end{lem}
\begin{proof}

Making the change of variables as $y= \Sigma_{k}^{-1/2}(x-\mu_{k})$, we write
\begin{align}
&\widetilde{H}[q]-H[q] \notag\\
& =
\sum_{k=1}^{K}\pi_{k} \int_{\mathbb{R}^{m}}  \mathcal{N}(x| \mu_{k}, \Sigma_{k}) \left\{\log \left(\sum_{k^{\prime}=1}^{K}\pi_{k^{\prime}} \mathcal{N}(x| \mu_{k^{\prime}}, \Sigma_{k}) \right) - \log \left(\pi_{k} \mathcal{N}(x| \mu_{k}, \Sigma_{k}) \right)  \right\}dx \hspace{30pt}\notag \\
&=\sum_{k=1}^{K}\pi_{k} \int_{\mathbb{R}^{m}}\frac{1}{\sqrt{(2\pi )^{m}|\Sigma_{k}|} } \exp\left(-\frac{\left\|x-\mu_{k}\right\|^{2}_{\Sigma_k}}{2}\right) \notag\\
&\hspace{60pt} \times \log \left(1+\sum_{k^{\prime}\neq k}\frac{\pi_{k^{\prime}} |\Sigma_k|^{\frac{1}{2}}}{\pi_{k}|\Sigma_{k^{\prime}}|^{\frac{1}{2}}} \exp\left(\frac{\left\|x-\mu_{k}\right\|^{2}_{\Sigma_k}-\left\|x-\mu_{k^{\prime}}\right\|^{2}_{\Sigma_{k^{\prime}}}}{2}\right) \right)dx \notag
\end{align}
\begin{align}
& =
\sum_{k=1}^{K}\pi_{k} \int_{\mathbb{R}^{m}}\frac{1}{\sqrt{(2\pi )^{m}}} \exp\left(-\frac{|y|^{2}}{2}\right)\notag\\
&\hspace{60pt} \times
\log \left(1+\sum_{k^{\prime}\neq k}\frac{\pi_{k^{\prime}} |\Sigma_k|^{\frac{1}{2}}}{\pi_{k}|\Sigma_{k^{\prime}}|^{\frac{1}{2}}} \exp\left(\frac{|y|^{2}-\left\|\Sigma_{k}^{\frac{1}{2}}\left(y-\Sigma_{k}^{-\frac{1}{2}}(\mu_{k^{\prime}}-\mu_{k})\right)\right\|^{2}_{\Sigma_{k^{\prime}}}}{2}\right) \right)dy.\\
\label{computation1}
\end{align}
Using the inequality $\log(1+ x) \leq \sqrt{x}$ $(x \geq 0)$ and the Cauchy-Schwarz inequality, we have
\begin{align} 
&\left|H[q] - \widetilde{H}[q] \right| \notag\\
& \leq
\sum_{k=1}^{K}\pi_{k} \int_{\mathbb{R}^{m}}\frac{1}{\sqrt{(2\pi )^{m}}} \exp\left(-\frac{|y|^{2}}{2}\right)
\sqrt{ \sum_{k^{\prime}\neq k}\frac{\pi_{k^{\prime}} |\Sigma_k|^{\frac{1}{2}}}{\pi_{k}|\Sigma_{k^{\prime}}|^{\frac{1}{2}}} \exp\left(\frac{|y|^{2}-\left\|\Sigma_{k}^{\frac{1}{2}}\left(y-\Sigma_{k}^{-\frac{1}{2}}(\mu_{k^{\prime}}-\mu_{k})\right)\right\|^{2}_{\Sigma_{k^{\prime}}}}{2} \right)}\ dy 
\label{computation2} \hspace{65pt}
\\
& =
\sum_{k=1}^{K} \int_{\mathbb{R}^{m}}\frac{1}{(2\pi)^{\frac{m}{4}}} \exp\left(-\frac{|y|^{2}}{4}\right)
\sqrt{ \frac{1}{(2\pi)^{\frac{m}{2}}} \sum_{k^{\prime}\neq k}\pi_{k}\pi_{k^{\prime}}\frac{ |\Sigma_k|^{\frac{1}{2}}}{|\Sigma_{k^{\prime}}|^{\frac{1}{2}}} \exp\left(\frac{-\left\|\Sigma_{k}^{\frac{1}{2}}\left(y-\Sigma_{k}^{-\frac{1}{2}}(\mu_{k^{\prime}}-\mu_{k})\right)\right\|^{2}_{\Sigma_{k^{\prime}}}}{2} \right)}\ dy \notag
& \leq
\sum_{k=1}^{K} \Biggl(\int_{\mathbb{R}^{m}}\frac{1}{(2\pi)^{\frac{m}{2}}} \exp\left(-\frac{|y|^{2}}{2}\right) dy \Biggr)^{\frac{1}{2}}\notag\\
&\hspace{60pt} \times
\underbrace{
\left(\int_{\mathbb{R}^{m}} \frac{1}{(2\pi)^{\frac{m}{2}}} \sum_{k^{\prime}\neq k}\pi_{k}\pi_{k^{\prime}}\frac{ |\Sigma_k|^{\frac{1}{2}}}{|\Sigma_{k^{\prime}}|^{\frac{1}{2}}} \exp\left(\frac{-\left\|\Sigma_{k}^{\frac{1}{2}}\left(y-\Sigma_{k}^{-\frac{1}{2}}(\mu_{k^{\prime}}-\mu_{k})\right)\right\|^{2}_{\Sigma_{k^{\prime}}}}{2} \right)dy \right)^{\frac{1}{2}}
}_
{\displaystyle=\left(\left(\sum_{k^{\prime}\neq k}\pi_{k}\pi_{k^{\prime}}\right) \int_{\mathbb{R}^{m}}\frac{1}{(2\pi)^{\frac{m}{2}}}\exp\left(-\frac{|z|^{2}}{2}\right)dz\right)^{\frac{1}{2}}} \notag\\
&=\sum_{k=1}^{K}\left(\sum_{k^{\prime}\neq k}\pi_{k}\pi_{k^{\prime}}\right)^{\frac{1}{2}}
=\sum_{k=1}^{K}\sqrt{\pi_k(1-\pi_k)} \leq \sum_{k=1}^{K}\frac{\pi_k + (1-\pi_k)}{2} \leq \frac{K}{2}. 
\end{align}
Thus, the proof of Lemma \ref{part1} is finished.
\end{proof}
\newpage
\begin{lem}\label{part2}
Let $s \in (0,1)$. Then
\begin{equation}
\left|H[q] - \widetilde{H}[q] \right| \leq \frac{2}{(1-s)^{\frac{m}{4}}}\sum_{k=1}^{K}\sum_{k^{\prime}\neq k} \sqrt{\pi_k\pi_{k^{\prime}}} \exp\left(-\frac{s\alpha_{k,k^{\prime}}^{2}}{4} \right). \label{estimate2}
\end{equation}
\end{lem}
\begin{proof}
Using the inequality \eqref{computation2} and the inequality $\sqrt{\sum_{i}a_i} \leq \sum_{i} \sqrt{a_i}$, we decompose
\begin{equation*}
\begin{split}
&\left|H[q] - \widetilde{H}[q] \right| \\
&\leq
\sum_{k=1}^{K}
\sum_{k^{\prime}\neq k}\sqrt{\pi_k\pi_{k^{\prime}}
\frac{ |\Sigma_k|^{\frac{1}{2}}}{|\Sigma_{k^{\prime}}|^{\frac{1}{2}}}}
\int_{\mathbb{R}^{m}}\frac{1}{\sqrt{(2\pi )^{m}}} \exp\left(-\frac{|y|^{2}}{4}- \frac{\left\|\Sigma_{k}^{\frac{1}{2}}\left(y-\Sigma_{k}^{-\frac{1}{2}}(\mu_{k^{\prime}}-\mu_{k})\right)\right\|^{2}_{\Sigma_{k^{\prime}}}}{4} \right) dy\\
& =
\sum_{k=1}^{K}
\sum_{k^{\prime}\neq k}\sqrt{\pi_k\pi_{k^{\prime}}
\frac{ |\Sigma_k|^{\frac{1}{2}}}{|\Sigma_{k^{\prime}}|^{\frac{1}{2}}}}
\int_{|y|<\alpha_{k,k^{\prime}}}\frac{1}{\sqrt{(2\pi )^{m}}} \exp\left(-\frac{|y|^{2}}{4}- \frac{\left\|\Sigma_{k}^{\frac{1}{2}}\left(y-\Sigma_{k}^{-\frac{1}{2}}(\mu_{k^{\prime}}-\mu_{k})\right)\right\|^{2}_{\Sigma_{k^{\prime}}}}{4} \right) dy \\
& \quad +
\sum_{k=1}^{K}
\sum_{k^{\prime}\neq k}\sqrt{\pi_k\pi_{k^{\prime}}
\frac{ |\Sigma_k|^{\frac{1}{2}}}{|\Sigma_{k^{\prime}}|^{\frac{1}{2}}}}
\int_{|y|>\alpha_{k,k^{\prime}}}\frac{1}{\sqrt{(2\pi )^{m}}} \exp\left(-\frac{|y|^{2}}{4}- \frac{\left\|\Sigma_{k}^{\frac{1}{2}}\left(y-\Sigma_{k}^{-\frac{1}{2}}(\mu_{k^{\prime}}-\mu_{k})\right)\right\|^{2}_{\Sigma_{k^{\prime}}}}{4} \right) dy \\
& \eqqcolon D^{i} + D^{o}. 
\end{split}
\end{equation*}
Firstly, we evaluate the term $D^{i}$.
By the definition of $\alpha_{k,k'}$, we have
\begin{equation*}
    \left|\Sigma_{k}^{-\frac{1}{2}}(\mu_{k^{\prime}}-\mu_{k})\right|=\alpha_{k,k'}\left(1+\left\|\Sigma_{k}^{-\frac{1}{2}}\Sigma_{k'}^{\frac{1}{2}}\right\|_{\rm op}\right)>|y|+\alpha_{k,k'}\left\|\Sigma_{k}^{-\frac{1}{2}}\Sigma_{k'}^{\frac{1}{2}}\right\|_{\rm op}.
\end{equation*}
Then it follows from properties of $\|\cdot\|_{\Sigma_{k'}}, \|\cdot\|_{\rm op}$, and triangle inequality that
\begin{align} \label{Di_estimate}
    \left\|\Sigma_{k}^{\frac{1}{2}}\left(y-\Sigma_{k}^{-\frac{1}{2}}(\mu_{k^{\prime}}-\mu_{k})\right)\right\|_{\Sigma_{k^{\prime}}} 
    &=\left|\Sigma_{k'}^{-\frac{1}{2}}\Sigma_k^{\frac{1}{2}}\left(y-\Sigma_{k}^{-\frac{1}{2}}(\mu_{k^{\prime}}-\mu_{k})\right)\right|
    \notag \\
    &\geq \frac{\left|y-\Sigma_{k}^{-\frac{1}{2}}(\mu_{k^{\prime}}-\mu_{k})\right|}{\left\|\left(\Sigma_{k'}^{-\frac{1}{2}}\Sigma_{k}^{\frac{1}{2}}\right)^{-1}\right\|_{\rm op}}
    \geq \frac{\left|\Sigma_{k}^{-\frac{1}{2}}(\mu_{k^{\prime}}-\mu_{k})\right|-|y|}{\left\|\Sigma_{k}^{-\frac{1}{2}}\Sigma_{k'}^{\frac{1}{2}}\right\|_{\rm op}}
    >\alpha_{k,k'}.
\end{align}
From inequality~\eqref{Di_estimate} and the Cauchy-Schwarz inequality, it follows that for $s \in (0,1)$,
\begin{align*}
 D^{i} &=
\sum_{k=1}^{K}
\sum_{k^{\prime}\neq k}\sqrt{\pi_k\pi_{k^{\prime}}
\frac{ |\Sigma_k|^{\frac{1}{2}}}{|\Sigma_{k^{\prime}}|^{\frac{1}{2}}}}
\int_{|y|<\alpha_{k,k^{\prime}}}\frac{1}{\sqrt{(2\pi )^{m}}}\\
&\qquad \times \hspace{0cm}
\exp\left(-\frac{|y|^{2}}{4}- \frac{s\left\|\Sigma_{k}^{\frac{1}{2}}\left(y-\Sigma_{k}^{-\frac{1}{2}}(\mu_{k^{\prime}}-\mu_{k})\right)\right\|^{2}_{\Sigma_{k^{\prime}}}}{4}-\frac{(1-s)\left\|\Sigma_{k}^{\frac{1}{2}}\left(y-\Sigma_{k}^{-\frac{1}{2}}(\mu_{k^{\prime}}-\mu_{k})\right)\right\|^{2}_{\Sigma_{k^{\prime}}}}{4} \right) dy \\
&\leq
\sum_{k=1}^{K}
\sum_{k^{\prime}\neq k} 
\sqrt{\pi_k\pi_{k^{\prime}}}
\exp\left(-\frac{s \alpha^{2}_{k,k^{\prime}}}{4}\right)
\int_{|y|<\alpha_{k,k^{\prime}}}\frac{1}{(2\pi )^{\frac{m}{4}}} \\
& \hspace{60pt} \times
\exp\left(-\frac{|y|^{2}}{4}\right)
\sqrt{\frac{ |\Sigma_k|^{\frac{1}{2}}}{|\Sigma_{k^{\prime}}|^{\frac{1}{2}}}}
\frac{1}{(2\pi )^{\frac{m}{4}}}
\exp\left(-\frac{(1-s)\left\|\Sigma_{k}^{\frac{1}{2}}\left(y-\Sigma_{k}^{-\frac{1}{2}}(\mu_{k^{\prime}}-\mu_{k})\right)\right\|^{2}_{\Sigma_{k^{\prime}}}}{4} \right) dy \hspace{25pt}\\
&\leq
\sum_{k=1}^{K}
\sum_{k^{\prime}\neq k} 
\sqrt{\pi_k\pi_{k^{\prime}}}
\exp\left(-\frac{s \alpha^{2}_{k,k^{\prime}}}{4}\right)
\left(\int_{\mathbb{R}^{m}}\frac{1}{(2\pi )^{\frac{m}{2}}}\exp\left(-\frac{|y|^{2}}{2}\right)dy \right)^{\frac{1}{2}} \\
& 
\hspace{60pt} \times
\left(\int_{\mathbb{R}^{m}}
\frac{ |\Sigma_k|^{\frac{1}{2}}}{|\Sigma_{k^{\prime}}|^{\frac{1}{2}}}
\frac{1}{(2\pi )^{\frac{m}{2}}}
\exp\left(-\frac{(1-s)\left\|\Sigma_{k}^{\frac{1}{2}}\left(y-\Sigma_{k}^{-\frac{1}{2}}(\mu_{k^{\prime}}-\mu_{k})\right)\right\|^{2}_{\Sigma_{k^{\prime}}}}{2} \right) dy \right)^{\frac{1}{2}}\\
&
=
\sum_{k=1}^{K}
\sum_{k^{\prime}\neq k}
\sqrt{\pi_k\pi_{k^{\prime}}}
\exp\left(-\frac{s \alpha^{2}_{k,k^{\prime}}}{4}\right) 
\left(\int_{\mathbb{R}^{m}}\frac{1}{(2\pi )^{\frac{m}{2}}}\exp\left(-\frac{(1-s)|z|^{2}}{2}\right)dz\right)^{\frac{1}{2}}\\
&=
\sum_{k=1}^{K}
\sum_{k^{\prime}\neq k}
\sqrt{\pi_k\pi_{k^{\prime}}}
\exp\left(-\frac{s \alpha^{2}_{k,k^{\prime}}}{4}\right)
(1-s)^{-\frac{m}{4}}\underbrace{\left(\int_{\mathbb{R}^{m}}\frac{1}{(2\pi )^{\frac{m}{2}}|(1-s)^{-1}I|^{\frac{1}{2}}}\exp\left(-\frac{1}{2}\|z\|_{(1-s)^{-1}I}^{2}\right)dz\right)^{\frac{1}{2}}}_{\displaystyle=1}\\
&=\frac{1}{(1-s)^{\frac{m}{4}}}
\sum_{k=1}^{K}
\sum_{k^{\prime}\neq k}
\sqrt{\pi_k\pi_{k^{\prime}}}
\exp\left(-\frac{s \alpha^{2}_{k,k^{\prime}}}{4}\right),
\label{D_int}
\end{align*}
where we have used the change of variable as $z = \Sigma_{k^{\prime}}^{-\frac{1}{2}}\Sigma_{k}^{\frac{1}{2}}\left(y-\Sigma_{k}^{-\frac{1}{2}}(\mu_{k^{\prime}}-\mu_{k})\right)$.

Secondly, we evaluate the term $D^{o}$. 
In the same way as above, we have
\begin{equation*}
\begin{split} 
D^{o} &=
\sum_{k=1}^{K}
\sum_{k^{\prime}\neq k}\sqrt{\pi_k\pi_{k^{\prime}}
\frac{ |\Sigma_k|^{\frac{1}{2}}}{|\Sigma_{k^{\prime}}|^{\frac{1}{2}}}}\\ 
& \hspace{20pt} \times\int_{|y|>\alpha_{k,k^{\prime}}}\frac{1}{\sqrt{(2\pi )^{m}}}
\exp\left(-\frac{s|y|^{2}}{4}-\frac{(1-s)|y|^{2}}{4}
- \frac{\left\|\Sigma_{k}^{\frac{1}{2}}\left(y-\Sigma_{k}^{-\frac{1}{2}}(\mu_{k^{\prime}}-\mu_{k})\right)\right\|^{2}_{\Sigma_{k^{\prime}}}}{4} \right) dy\\
&\leq
\sum_{k=1}^{K}
\sum_{k^{\prime}\neq k}
\sqrt{\pi_k\pi_{k^{\prime}}}
\exp\left(-\frac{s \alpha^{2}_{k,k^{\prime}}}{4}\right) 
\int_{|y|>\alpha_{k,k^{\prime}}}\frac{1}{(2\pi )^{\frac{m}{4}}}\exp\left(-\frac{(1-s)|y|^{2}}{4}\right) \\
& 
\quad \hspace{3cm} \times
\sqrt{\frac{ |\Sigma_k|^{\frac{1}{2}}}{|\Sigma_{k^{\prime}}|^{\frac{1}{2}}}}
\frac{1}{(2\pi )^{\frac{m}{4}}}
\exp\left(-\frac{\left\|\Sigma_{k}^{\frac{1}{2}}\left(y-\Sigma_{k}^{-\frac{1}{2}}(\mu_{k^{\prime}}-\mu_{k})\right)\right\|^{2}_{\Sigma_{k^{\prime}}}}{4} \right) dy \\
&\leq
\sum_{k=1}^{K}
\sum_{k^{\prime}\neq k} 
\sqrt{\pi_k\pi_{k^{\prime}}}
\exp\left(-\frac{s \alpha^{2}_{k,k^{\prime}}}{4}\right)
\left(\int_{\mathbb{R}^{m}}
\frac{1}{(2\pi )^{\frac{m}{2}}}\exp\left(-\frac{(1-s)|y|^{2}}{2}\right)dy \right)^{\frac{1}{2}} \\
& 
\quad \hspace{2cm} \times
\left(\int_{\mathbb{R}^{m}}
\frac{ |\Sigma_k|^{\frac{1}{2}}}{|\Sigma_{k^{\prime}}|^{\frac{1}{2}}}
\frac{1}{(2\pi )^{\frac{m}{2}}}
\exp\left(-\frac{\left\|\Sigma_{k}^{\frac{1}{2}}\left(y-\Sigma_{k}^{-\frac{1}{2}}(\mu_{k^{\prime}}-\mu_{k})\right)\right\|^{2}_{\Sigma_{k^{\prime}}}}{2} \right) dy \right)^{\frac{1}{2}}\\
&
=
\sum_{k=1}^{K}
\sum_{k^{\prime}\neq k}
\sqrt{\pi_k\pi_{k^{\prime}}}
\exp\left(-\frac{s \alpha^{2}_{k,k^{\prime}}}{4}\right) 
\left(\int_{\mathbb{R}^{m}}\frac{1}{(2\pi )^{\frac{m}{2}}}\exp\left(-\frac{(1-s)|y|^{2}}{2}\right)dy \right)^{\frac{1}{2}} \\
&
= 
\frac{1}{(1-s)^{\frac{m}{4}}}
\sum_{k=1}^{K}
\sum_{k^{\prime}\neq k}
\sqrt{\pi_k\pi_{k^{\prime}}}
\exp\left(-\frac{s \alpha^{2}_{k,k^{\prime}}}{4}\right).
\end{split}
\label{D_out}
\end{equation*}
Combining the estimates obtained now, we conclude \eqref{estimate2}.
\end{proof}
\begin{lem}
\label{part2.1}
Let $s \in (0,1)$. Then
\begin{equation*}
\left| H[q] - \widetilde{H}[q] \right|\leq \frac{2}{(1-s)^{\frac{m}{4}}}\sum_{k=1}^{K}\sum_{k^{\prime}\neq k} \sqrt{\pi_k\pi_{k^{\prime}}} \exp\left(-\frac{s\alpha_{\{k,k'\}}^{2}}{4} \right).
\end{equation*}
\end{lem}
\begin{proof}
The proof is almost the same as Theorem \ref{part2} except the evaluation \eqref{Di_estimate}.
By the change of variables $y=\Sigma_k^{-\frac{1}{2}}(x-\mu_k)$, we have
\begin{align*}
    \|x-\mu_k\|_{\Sigma_k}&=\left|\Sigma_k^{-\frac{1}{2}}(x-\mu_k)\right|=|y|, 
    \\
    \|x-\mu_{k'}\|_{\Sigma_{k'}}&=\left|\Sigma_{k'}^{-\frac{1}{2}}(x-\mu_{k'})\right|=\left|\Sigma_{k'}^{-\frac{1}{2}}\Sigma_k^{\frac{1}{2}}\left(y-\Sigma_k^{-\frac{1}{2}}(\mu_{k'}-\mu_k)\right)\right|=\left\|\Sigma_k^{\frac{1}{2}}\left(y-\Sigma_k^{-\frac{1}{2}}(\mu_{k'}-\mu_k)\right)\right\|_{\Sigma_{k'}}.
\end{align*}
From the definition of $\alpha_{\{k,k'\}}$,
\begin{equation*}
\{y \in \mathbb{R}^m:|y|<\alpha \} \cap 
\left\{
y \in \mathbb{R}^m:\left\|\Sigma_k^{\frac{1}{2}}\left(y-\Sigma_k^{-\frac{1}{2}}(\mu_{k'}-\mu_k)\right)\right\|_{\Sigma_{k'}} < \alpha
\right\} = \varnothing,
\end{equation*}
then if $|y|<\alpha_{\{k,k'\}}$, we obtain
$$\left\|\Sigma_k^{\frac{1}{2}}\left(y-\Sigma_k^{-\frac{1}{2}}(\mu_{k'}-\mu_k)\right)\right\|_{\Sigma_{k'}} \geq \alpha_{\{k,k'\}}.$$
The proof complete by replacing $\alpha_{k,k'}$ with $\alpha_{\{k,k'\}}$ in the proof of Lemma~\ref{part2}.
\end{proof}

\begin{lem} \label{remark_alpha}
$\alpha_{\{k,k'\}} \geq \alpha_{k,k'}$
for any $k,k' \in [K]$.
\end{lem}
\begin{proof}
    When $x$ satisfies
    \begin{equation*}
        \|x-\Sigma_k^{-\frac{1}{2}}\mu_{k'}\|_{\Sigma_k^{-1}\Sigma_{k'}}<\alpha_{k,k'},
    \end{equation*}
    because
    \begin{equation*}
    \|x-\Sigma_k^{-\frac{1}{2}}\mu_{k'}\|_{\Sigma_k^{-1}\Sigma_{k'}}=\left|\Sigma_{k'}^{-\frac{1}{2}}\Sigma_k^{\frac{1}{2}}\left(x-\Sigma_k^{-\frac{1}{2}}\mu_{k'}\right)\right| \geq \frac{|x-\Sigma_k^{-\frac{1}{2}}\mu_{k'}|}{\left\|\left(\Sigma_{k'}^{-\frac{1}{2}}\Sigma_k^{\frac{1}{2}}\right)^{-1}\right\|_{\bf op}}=\frac{|x-\Sigma_k^{-\frac{1}{2}}\mu_{k'}|}{\sigma},
    \end{equation*}
    then $|x-\Sigma_k^{-\frac{1}{2}}\mu_{k'}| < \alpha_{k,k'}\sigma$, where $\sigma=\|\Sigma_k^{-\frac{1}{2}}\Sigma_{k'}^{\frac{1}{2}}\|_{\rm op}$.
    On the other hand, from the definition of $\alpha_{k,k'}$, we have
    \begin{equation*}
        \alpha_{k,k'}+\alpha_{k,k'}\sigma=|\Sigma_k^{-\frac{1}{2}}\mu_k-\Sigma_k^{\frac{1}{2}}\mu_{k'}|,
    \end{equation*}
    and thus $\{x \in \mathbb{R}^m:|x-\Sigma_k^{-\frac{1}{2}}\mu_k|<\alpha_{k,k'}\} \cap \{x \in \mathbb{R}^m:|x-\Sigma_k^{-\frac{1}{2}}\mu_{k'}|<\alpha_{k,k'}\sigma\}=\varnothing$.
    Therefore, we obtain
    \begin{equation*}
    \{x \in \mathbb{R}^m:|x-\Sigma_k^{-\frac{1}{2}}\mu_k|<\alpha_{k,k'}\} \cap \{x \in \mathbb{R}^m:\|x-\Sigma_k^{-\frac{1}{2}}\mu_{k'}\|_{\Sigma_k^{-1}\Sigma_{k'}}<\alpha_{k,k'}\}=\varnothing.
    \end{equation*}
    Making the change of variables as $y=\Sigma_k^{\frac{1}{2}}x$,
    \begin{equation*}
    \{y \in \mathbb{R}^m:\|y-\mu_{k'}\|_{\Sigma_k}<\alpha_{k,k'}\} \cap \{y \in \mathbb{R}^m:\|y-\mu_{k'}\|_{\Sigma_{k'}}<\alpha_{k,k'}\}=\varnothing.
    \end{equation*}
    From definition~\eqref{alpha_(k)}, $\alpha_{\{k,k'\}}\geq\alpha_{k,k'}$ is obtained.
\end{proof}

\begin{lem}\label{part2.5}
\[
\begin{split}
&\sum_{k=1}^{K}\sum_{k^{\prime}\neq k}
\frac{\pi_k\pi_{k^{\prime}}}{1-\pi_k}
c_{k,k^{\prime}}\log
\left(
1+\frac{1-\pi_k}{\pi_k}
\frac{|\Sigma_{k}|^{\frac{1}{2}}}{\displaystyle \max_{l}|\Sigma_{l}|^{\frac{1}{2}}}
 \exp
\left(
-\frac{
\Bigl(1+\|\Sigma_{k^{\prime}}^{-\frac{1}{2}}\Sigma_{k}^{\frac{1}{2}} \|_{\rm op}
\Bigr)^2
}{2}
\alpha_{k^{\prime},k}^{2}
\right)
\right) 
\leq 
\left|H[q] - \widetilde{H}[q] \right|,
\end{split}
\]
where the coefficient $c_{k,k^{\prime}}$ is defined by
\[
c_{k,k^{\prime}}\coloneqq\frac{1}{\sqrt{(2\pi)^m}}\int_{\mathbb{R}^{m}_{k,k^{\prime}}}\exp\left(-\frac{|y|^2}{2}\right)dy \geq 0,
\]
and the set $\mathbb{R}^{m}_{k,k^{\prime}}$ is defined by
\begin{equation}\label{app:definition}
\mathbb{R}^{m}_{k,k^{\prime}}
\coloneqq\left\{ y \in \mathbb{R}^m : 
\begin{array}{cc}
y\cdot y \geq (\Sigma_{k}^{\frac{1}{2}}\Sigma_{k^{\prime}}^{-1}\Sigma_{k}^{\frac{1}{2}}y)\cdot y,
\\
y\cdot (\Sigma_{k}^{\frac{1}{2}}\Sigma_{k^{\prime}}^{-1} (\mu_{k^{\prime}} -\mu_k) ) \geq 0
\end{array}
\right\}.
\end{equation}
\end{lem}
\begin{proof}
Using the equality in \eqref{computation1}, we write
\[
\begin{split}
&\left|H[q] - \widetilde{H}[q] \right|\\
& =
\sum_{k=1}^{K}\pi_{k} \int_{\mathbb{R}^{m}}\frac{1}{\sqrt{(2\pi )^{m}}} \exp\left(-\frac{|y|^{2}}{2}\right) \\
&\hspace{40pt} \times
\log \left(1+\sum_{k^{\prime}\neq k}\frac{\pi_{k^{\prime}} |\Sigma_k|^{\frac{1}{2}}}{\pi_{k}|\Sigma_{k^{\prime}}|^{\frac{1}{2}}} \exp\left(\frac{|y|^{2}-\left\|\Sigma_{k}^{\frac{1}{2}}\left(y-\Sigma_{k}^{-\frac{1}{2}}(\mu_{k^{\prime}}-\mu_{k})\right)\right\|^{2}_{\Sigma_{k^{\prime}}}}{2}\right) \right)dy
\\
&
\geq 
\sum_{k=1}^{K}\pi_{k} \int_{\mathbb{R}^{m}}\frac{1}{\sqrt{(2\pi )^{m}}} \exp\left(-\frac{|y|^{2}}{2}\right)
\\
&
\hspace{1cm} \times
\log \left(1+
\frac{1-\pi_k}{\pi_k}
\frac{|\Sigma_k|^{\frac{1}{2}}}{\displaystyle \max_l|\Sigma_{l}|^{\frac{1}{2}}}
\sum_{k^{\prime}\neq k}
\frac{\pi_{k^{\prime}}}{1-\pi_{k}}
\exp\left(\frac{|y|^{2}-\left\|\Sigma_{k}^{\frac{1}{2}}\left(y-\Sigma_{k}^{-\frac{1}{2}}(\mu_{k^{\prime}}-\mu_{k})\right)\right\|^{2}_{\Sigma_{k^{\prime}}}}{2}\right) \right)dy.
\end{split}
\]
Since $\log(1+\lambda x)$ is a concave function of $x>0$ for $\lambda>0$, we estimate that
\begin{align*}
\left|H[q] - \widetilde{H}[q] \right|
&
\geq 
\sum_{k=1}^{K}\pi_{k} \int_{\mathbb{R}^{m}}\frac{1}{\sqrt{(2\pi )^{m}}} \exp\left(-\frac{|y|^{2}}{2}\right)
\sum_{k^{\prime}\neq k}
\frac{\pi_{k^{\prime}}}{1-\pi_{k}}
\\
&
\hspace{1cm} \times
\log \left(1+
\frac{1-\pi_k}{\pi_k}
\frac{|\Sigma_k|^{\frac{1}{2}}}{\displaystyle \max_l|\Sigma_{l}|^{\frac{1}{2}}}
\exp\left(\frac{|y|^{2}-\left\|\Sigma_{k}^{\frac{1}{2}}\left(y-\Sigma_{k}^{-\frac{1}{2}}(\mu_{k^{\prime}}-\mu_{k})\right)\right\|^{2}_{\Sigma_{k^{\prime}}}}{2}\right) \right)dy\\
\hspace{30pt}&\geq
\sum_{k=1}^{K}\pi_{k} \int_{\mathbb{R}^{m}_{k,k^{\prime}}}\frac{1}{\sqrt{(2\pi )^{m}}} \exp\left(-\frac{|y|^{2}}{2}\right)
\sum_{k^{\prime}\neq k}
\frac{\pi_{k^{\prime}}}{1-\pi_{k}}
\\
&
\hspace{1cm} \times
\log \left(1+
\frac{1-\pi_k}{\pi_k}
\frac{|\Sigma_k|^{\frac{1}{2}}}{\displaystyle \max_l|\Sigma_{l}|^{\frac{1}{2}}}
\exp\left(\frac{|y|^{2}-\left\|\Sigma_{k}^{\frac{1}{2}}\left(y-\Sigma_{k}^{-\frac{1}{2}}(\mu_{k^{\prime}}-\mu_{k})\right)\right\|^{2}_{\Sigma_{k^{\prime}}}}{2}\right) \right)dy.
\label{lowerapp1}
\end{align*}
Here, it follows from the two conditions in the definition \eqref{app:definition} of $\mathbb{R}^{m}_{k,k^{\prime}}$ that 
\begin{equation*}
\label{est:alpha_k}
\begin{split}
 |y|^{2}-\left\|\Sigma_{k}^{\frac{1}{2}}\left(y-\Sigma_{k}^{-\frac{1}{2}}(\mu_{k^{\prime}}-\mu_{k})\right)\right\|^{2}_{\Sigma_{k^{\prime}}}
&\geq
\left|\Sigma_{k^{\prime}}^{-\frac{1}{2}}\Sigma_{k}^{\frac{1}{2}}y\right|^{2}
-
\left|\Sigma_{k^{\prime}}^{-\frac{1}{2}}\Sigma_{k}^{\frac{1}{2}}y-\Sigma_{k^{\prime}}^{-\frac{1}{2}}(\mu_{k^{\prime}}-\mu_{k})\right|^{2}
\\
&
= -\left|\Sigma_{k^{\prime}}^{-\frac{1}{2}}(\mu_{k^{\prime}}-\mu_{k})\right|^{2}
+2 y \cdot \left(\Sigma_{k}^{\frac{1}{2}}\Sigma_{k^{\prime}}^{-1}(\mu_{k^{\prime}}-\mu_{k}) \right) 
\\
&
\geq 
-\left|\Sigma_{k^{\prime}}^{-\frac{1}{2}}(\mu_{k^{\prime}}-\mu_{k})\right|^{2} \\
&= - \left(1+\|\Sigma_{k^{\prime}}^{-\frac{1}{2}}\Sigma_{k}^{\frac{1}{2}} \|_{\rm op}\right)^{2}\alpha_{k^{\prime},k}^{2},
\end{split}
\end{equation*}
for $y \in \mathbb{R}^{m}_{k,k^{\prime}}$, 
where we used the cosine formula in the second step. 
Combining the above estimates, we have
\begin{align*}
\left|H[q] - \widetilde{H}[q] \right|
&\geq 
\sum_{k=1}^{K}\pi_{k} \int_{\mathbb{R}^{m}_{k,k^{\prime}}}\frac{1}{\sqrt{(2\pi )^{m}}} \exp\left(-\frac{|y|^{2}}{2}\right)
\sum_{k^{\prime}\neq k}
\frac{\pi_{k^{\prime}}}{1-\pi_{k}} \\
&\hspace{60pt} \times
\log \left(1+
\frac{1-\pi_k}{\pi_k}
\frac{|\Sigma_k|^{\frac{1}{2}}}{\displaystyle\max_l|\Sigma_{l}|^{\frac{1}{2}}}
\exp
\left(
-
\frac{
\left(1+\|\Sigma_{k^{\prime}}^{-\frac{1}{2}}
\Sigma_{k}^{\frac{1}{2}} \|_{\rm op}\right)^{2}
}{2}
\alpha_{k^{\prime},k}^{2}
\right)
\right)dy\\
&=\sum_{k=1}^{K}\sum_{k^{\prime}\neq k}
\frac{\pi_k\pi_{k^{\prime}}}{1-\pi_k}
c_{k,k^{\prime}}\log
\left(
1+\frac{1-\pi_k}{\pi_k}
\frac{|\Sigma_{k}|^{\frac{1}{2}}}{\displaystyle \max_{l}|\Sigma_{l}|^{\frac{1}{2}}}
 \exp
\left(
-\frac{
\Bigl(1+\|\Sigma_{k^{\prime}}^{-\frac{1}{2}}\Sigma_{k}^{\frac{1}{2}} \|_{\rm op}
\Bigr)^2
}{2}
\alpha_{k^{\prime},k}^{2}
\right)
\right).
\end{align*}
The proof of Lemma \ref{part2.5} is finished.
\end{proof}
\begin{rem}\label{remarkA}
Either $c_{k, k^{\prime}}$ or $c_{k^{\prime},k}$ is positive.
Indeed,
\begin{itemize}

\item 
if $\Sigma_{k}^{-1} -\Sigma_{k^{\prime}}^{-1}$ has at least one positive eigenvalue, then $c_{k,k^{\prime}}$ is positive;

\item
if all eigenvalues of $\Sigma_{k}^{-1} -\Sigma_{k^{\prime}}^{-1} \neq O$ are non-positive, then $\Sigma_{k^{\prime}}^{-1}-\Sigma_{k}^{-1}$ has at least one positive eigenvalue;

\item
if $\Sigma_{k}^{-1}-\Sigma_{k^{\prime}}^{-1}=O$, then $c_{k, k^{\prime}}=1/2$ because $\mathbb{R}^{m}_{k, k^{\prime}}$ is a half-space of $\mathbb{R}^{m}$,

\end{itemize}
where $O$ is the zero matrix.
\end{rem}

\subsection{Proof of Corollary \ref{prob_ineq_cor}}\label{Proof of Corollary4.3}
We restate Corollary~\ref{prob_ineq_cor} as follows:
\begin{lem}
Let $c>0$.
Assume 
$\{\mu_k\}_k$ and $\{\Sigma_{k}\}_k$ such that 
\begin{equation}
\frac{\Sigma_{k}^{-\frac{1}{2}}(\mu_k-\mu_{k^{\prime}})}{1+\|\Sigma_k^{-\frac{1}{2}}\Sigma_{k^{\prime}}^{\frac{1}{2}} \|_{\rm op}} \sim \mathcal{N}(0, c^{2}I)
\label{random_mu_k_app}
\end{equation}
for all pairs $k, k^{\prime} \in [K]$ ($k\neq k^{\prime}$).
Then, for $\varepsilon>0$ and $s \in (0,1)$,
\begin{equation}
\begin{split}
&P\left(\left|H[q] - \widetilde{H}[q] \right| \geq \varepsilon \right) \leq \frac{2(K-1)}{\varepsilon}\left(\sqrt{1-s} \left(1+\frac{s c^2}{2}\right) \right)^{-\frac{m}{2}}.
\end{split}
\label{prob_ineq_gen_app}
\end{equation}
\end{lem}
\begin{proof}
Using Lemma \ref{part2} and Markov's inequality, for $\varepsilon>0$ and $s \in (0,1)$, we estimate
\[
P\left(\left|H[q] - \widetilde{H}[q] \right| \geq \varepsilon \right) \leq \frac{E\left[\left|H[q] - \widetilde{H}[q] \right|\right]}{\varepsilon} \leq \frac{2}{\varepsilon(1-s)^{\frac{m}{4}}}
\sum_{k=1}^{K}
\sum_{k^{\prime}\neq k}
\sqrt{\pi_k\pi_{k^{\prime}}}
E \left[ \exp\left(-\frac{s \alpha^{2}_{k,k^{\prime}}}{4}\right) \right].
\]
By the assumption \eqref{random_mu_k_app}, 
$\alpha^{2}_{k,k^{\prime}}/c^{2}$ follows the $\chi^{2}$-distribution with $m$ degrees of freedom, that is,
\[
\frac{1}{c^{2}} \left|\frac{\Sigma_{k}^{-\frac{1}{2}}(\mu_k-\mu_{k^{\prime}})}{1+\|\Sigma_k^{-\frac{1}{2}}\Sigma_{k^{\prime}}^{\frac{1}{2}} \|_{\rm op}} \right|^{2} \sim \chi_{m}^{2}.
\]
Therefore, we conclude from the moment-generating function for $\chi^{2}$-distribution that
\begin{align}
P\left(\left|H[q] - \widetilde{H}[q] \right| \geq \varepsilon \right)
&\leq 
\frac{2}{\varepsilon(1-s)^{\frac{m}{4}}}
\sum_{k=1}^{K}
\sum_{k^{\prime}\neq k}
\sqrt{\pi_k\pi_{k^{\prime}}}
E \left[ \exp\left(- \frac{sc^2}{4}
\frac{\alpha^{2}_{k,k^{\prime}}}{c^2}\right) \right] \nonumber
\\
&=
\frac{2}{\varepsilon(1-s)^{\frac{m}{4}}}
\sum_{k=1}^{K}
\sum_{k^{\prime}\neq k}
\sqrt{\pi_k\pi_{k^{\prime}}}
\left(1-2\left(- \frac{sc^2}{4}\right) \right)^{-\frac{m}{2}}\nonumber
\\
&\leq \frac{2(K-1)}{\varepsilon}\left( \sqrt{1-s} \left(1+\frac{s c^2}{2} \right) \right)^{-\frac{m}{2}}.\nonumber
\end{align}
\hspace{350pt}
\end{proof}

\subsection{Proof of Theorem~\ref{derivative-entropy}}\label{sec-proof of derivative entropy}
The next lemma is the same as Theorem~\ref{derivative-entropy}.
\begin{lem}
Let $k \in [K]$, $p,q \in [m]$, and $s \in (0,1)$. Then
\begin{align*}
\mathrm{(i)} 
&
\quad
\left| \frac{\partial}{\partial \mu_{k,p}} \left(H[q] - \widetilde{H}[q]\right) \right|
\leq
\frac{2}{(1-s)^{\frac{m+2}{4}}}
\sum_{k^{\prime}\neq k}
\sqrt{\pi_k \pi_{k^{\prime}}}
\left( \left\| \Gamma_{k^{\prime}}^{-1} \right\|_{1}+
\left\| \Gamma_{k}^{-1} \right\|_{1} \right)
\exp
\left(
-\frac{s\alpha_{k,k'}^{2}}{4}
\right),
\\
\hspace{-0.5cm}
\mathrm{(ii)} 
&
\quad
\left| \frac{\partial}{\partial \gamma_{k,pq}} \left(H[q] - \widetilde{H}[q]\right) \right| 
\notag \\
& \quad \leq
\frac{6}{(1-s)^{\frac{m+4}{4}}}
\sum_{k^{\prime}\neq k}
\sqrt{\pi_k \pi_{k^{\prime}}}
\left(
2|\Gamma_k|^{-1}|\Gamma_{k, pq}|
+
\left\| \Gamma_{k}^{-1} \right\|_{1}
+
\left\| \Gamma_{k'}^{-1} \right\|_{1}
\right)
\exp
\left(
-\frac{s\alpha_{k,k'}^{2}}{4}
\right)
\\ &\quad \text{for} \ \gamma_{k,pq}\in \mathbb{R} \ \text{satisfying} \ \|\Gamma_k^{-1}\|_1<\infty,\\
\mathrm{(iii)} 
&
\quad
\left| \frac{\partial}{\partial \pi_{k}} \left(H[q] - \widetilde{H}[q]\right) \right|
\leq
\frac{8}{(1-s)^{\frac{m}{4}}}
\sum_{k^{\prime}\neq k}
\sqrt{\frac{\pi_{k^{\prime}}}{\pi_k}}
\exp
\left(
-\frac{s\alpha_{k,k'}^{2}}{4}
\right),
\end{align*}
where $\mu_{k,p}$ and $\gamma_{k,pq}$ is the $p$-th and $(p,q)$-th components of vector $\mu_k$ and matrix $\Gamma_{k}$, respectively, and $\left\| \cdot \right\|_{1}$ is the entry-wise matrix $1$-norm, and $|\Gamma_{k,pq}|$ is the determinant of the $(m-1)\times(m-1)$ matrix that results from deleting $p$-th row and $q$-th column of matrix $\Gamma_{k}$.
Moreover, the same upper bounds hold for $\alpha_{\{k,k'\}}$ instead of $\alpha_{k,k'}$.
\end{lem}

\begin{proof}
In this proof, we denote $\Gamma_{k^{\prime}}^{-1}\Gamma_{k}\, y-\Gamma_{k^{\prime}}^{-1}(\mu_{k^{\prime}}-\mu_{k})$ by $\Theta(y\mid\mu_{k;k'},\Gamma_{k;k'})$.
From the equality \eqref{computation1}, we have
\begin{align}
&H[q] - \widetilde{H}[q] \notag\\
&=
\sum_{k=1}^{K}\pi_{k} \int_{\mathbb{R}^{m}}\frac{1}{\sqrt{(2\pi )^{m}}} \exp\left(-\frac{|y|^{2}}{2}\right)\
\log 
\left(
1+\sum_{k^{\prime}\neq k}\frac{\pi_{k^{\prime}} |\Gamma_k|}{\pi_{k}|\Gamma_{k^{\prime}}|} 
\exp\left(
\frac{|y|^{2}-|\Theta(y\mid\mu_{k;k'},\Gamma_{k;k'})|^{2}}{2}
\right)
\right)
dy \notag
\\
&
=
\pi_{k} \int_{\mathbb{R}^{m}}\frac{1}{\sqrt{(2\pi )^{m}}} \exp\left(-\frac{|y|^{2}}{2}\right)\
\log 
\left(
1+\sum_{k^{\prime}\neq k}\frac{\pi_{k^{\prime}} |\Gamma_k|}{\pi_{k}|\Gamma_{k^{\prime}}|} 
\exp\left(
\frac{|y|^{2}-|\Theta(y\mid\mu_{k;k'},\Gamma_{k;k'})|^{2}}{2}
\right)
\right)
dy \notag
\\
&
\hspace{20pt}+
\sum_{\ell \neq k}\pi_{\ell} \int_{\mathbb{R}^{m}}\frac{1}{\sqrt{(2\pi )^{m}}} \exp\left(-\frac{|y|^{2}}{2}\right)\notag\\
&\label{eq.A.3}\hspace{80pt}\times
\log 
\Biggl(
1+\hspace{-85pt}
\underbrace{
\sum_{k^{\prime} \neq \ell}
\frac{\pi_{k^{\prime}} |\Gamma_{\ell}|}{\pi_{\ell}|\Gamma_{k^{\prime}}|} 
\exp
\left(
\frac{|y|^{2}-|\Theta(y\mid\mu_{\ell;k'},\Gamma_{\ell;k'})|^{2}}{2}
\right)
}_
{\hspace{80pt}\displaystyle=\frac{\pi_{k} |\Gamma_{\ell}|}{\pi_{\ell}|\Gamma_{k}|} 
\exp
\left(
\frac{|y|^{2}-|\Theta(y\mid\mu_{\ell;k},\Gamma_{\ell;k})|^{2}}{2}
\right)
+\text{terms independent of $k$}}
\hspace{-85pt}\Biggr)
dy.
\end{align}
\noindent{\bf (i)}\ 
The derivatives of \eqref{eq.A.3} with respect to $\mu_{k,p}$ are calculated as
\begin{align}
&\frac{\partial}{\partial \mu_{k,p}} \left(H[q] - \widetilde{H}[q]\right) \notag\\
&
=
\pi_{k} \int_{\mathbb{R}^{m}}\frac{1}{\sqrt{(2\pi )^{m}}} \exp\left(-\frac{|y|^{2}}{2}\right)\\
&\hspace{40pt}\times
\frac{\displaystyle\sum_{k^{\prime}\neq k}\frac{\pi_{k^{\prime}} |\Gamma_k|}{\pi_{k}|\Gamma_{k^{\prime}}|} 
\exp\left(
\frac{|y|^{2}-|\Theta(y\mid\mu_{k;k'},\Gamma_{k;k'})|^{2}}{2}
\right)
\frac{\partial}{\partial \mu_{k,p}}
\left(  
\frac{-|\Theta(y\mid\mu_{k;k'},\Gamma_{k;k'})|^{2}}{2}
\right)
}
{\displaystyle1+\sum_{k^{\prime}\neq k}\frac{\pi_{k^{\prime}} |\Gamma_k|}{\pi_{k}|\Gamma_{k^{\prime}}|} 
\exp\left(
\frac{|y|^{2}-|\Theta(y\mid\mu_{k;k'},\Gamma_{k;k'})|^{2}}{2}
\right)}
dy \notag\\
&
\hspace{20pt}+
\sum_{\ell \neq k}\pi_{\ell} \int_{\mathbb{R}^{m}}\frac{1}{\sqrt{(2\pi )^{m}}} \exp\left(-\frac{|y|^{2}}{2}\right)\\
&\hspace{60pt}\times
\frac{\displaystyle\frac{\pi_{k} |\Gamma_{\ell}|}{\pi_{\ell}|\Gamma_{k}|} 
\exp
\left(
\frac{|y|^{2}-|\Theta(y\mid\mu_{\ell;k},\Gamma_{\ell;k})|^{2}}{2}
\right)
\frac{\partial}{\partial \mu_{k,p}}
\left(  
\frac{-|\Theta(y\mid\mu_{\ell;k},\Gamma_{\ell;k})|^{2}}{2}
\right)
}
{\displaystyle1+
\sum_{k^{\prime} \neq \ell}
\frac{\pi_{k^{\prime}} |\Gamma_{\ell}|}{\pi_{\ell}|\Gamma_{k^{\prime}}|}
\exp
\left(
\frac{|y|^{2}-|\Theta(y\mid\mu_{\ell;k'},\Gamma_{\ell;k'})|^{2}}{2}
\right)
}
dy.\\
\label{derivative-mu-1}
\end{align}

Since $\frac{x}{a+x} \leq x^{\frac{1}{2}}$ for $x>1$ when $a\geq 1$, we estimate
\begin{equation}
\frac{\displaystyle\frac{\pi_{k^{\prime}} |\Gamma_k|}{\pi_{k}|\Gamma_{k^{\prime}}|} 
\exp\left(
\frac{|y|^{2}-|\Theta(y\mid\mu_{k;k'},\Gamma_{k;k'})|^{2}}{2}
\right)
}
{\displaystyle1+\sum_{k^{\prime}\neq k}\frac{\pi_{k^{\prime}} |\Gamma_k|}{\pi_{k}|\Gamma_{k^{\prime}}|} 
\exp\left(
\frac{|y|^{2}-|\Theta(y\mid\mu_{k;k'},\Gamma_{k;k'})|^{2}}{2}
\right)}
\leq
\sqrt{\frac{\pi_{k^{\prime}} |\Gamma_k|}{\pi_{k}|\Gamma_{k^{\prime}}|}}
\exp\left(
\frac{|y|^{2}-|\Theta(y\mid\mu_{k;k'},\Gamma_{k;k'})|^{2}}{4}
\right),
\label{derivative-mu-2}\noeqref{derivative-mu-2}
\end{equation}
and
\begin{equation}
\frac{\displaystyle\frac{\pi_{k} |\Gamma_{\ell}|}{\pi_{\ell}|\Gamma_{k}|} 
\exp
\left(
\frac{|y|^{2}-|\Theta(y\mid\mu_{\ell;k},\Gamma_{\ell;k})|^{2}}{2}
\right)
}
{\displaystyle1+
\sum_{k^{\prime} \neq \ell}
\frac{\pi_{k^{\prime}} |\Gamma_{\ell}|}{\pi_{\ell}|\Gamma_{k^{\prime}}|}
\exp
\left(
\frac{|y|^{2}-|\Theta(y\mid\mu_{\ell;k'},\Gamma_{\ell;k'})|^{2}}{2}
\right)
}\leq
\sqrt{\frac{\pi_{k} |\Gamma_{\ell}|}{\pi_{\ell}|\Gamma_{k}|}}
\exp
\left(
\frac{|y|^{2}-|\Theta(y\mid\mu_{\ell;k},\Gamma_{\ell;k})|^{2}}{4}
\right).
\label{derivative-mu-3}\noeqref{derivative-mu-3}
\end{equation}

We also calculate
\begin{align}
\frac{\partial}{\partial \mu_{k,p}}
\left(  
\frac{-|\Theta(y\mid\mu_{k;k'},\Gamma_{k;k'})|^{2}}{2}
\right)
&=
\sum_{i=1}^{m}\left[\Theta(y\mid\mu_{k;k'},\Gamma_{k;k'})\right]_{i} \gamma_{k^{\prime},ip}^{-1},
\label{derivative-mu-4}\noeqref{derivative-mu-4}
\end{align}
and
\begin{align}
\frac{\partial}{\partial \mu_{k,p}}
\left(
\frac{-|\Theta(y\mid\mu_{\ell;k},\Gamma_{\ell;k})|^{2}}{2}
\right)
&=
\sum_{i=1}^{m}\left[\Theta(y\mid\mu_{\ell;k},\Gamma_{\ell;k})\right]_{i} \gamma_{k,ip}^{-1},
\label{derivative-mu-5}
\end{align}
where we denote by $[v]_{i}$ the $i$-th component of vector $v$, and $\gamma_{k^{\prime},ip}^{-1}$ and $\gamma_{k,ip}^{-1}$ the $(i,p)$-th component of matrix $\Gamma_{k^{\prime}}^{-1}$ and $\Gamma_{k}^{-1}$, respectively.

By using \eqref{derivative-mu-1}--\eqref{derivative-mu-5}, 
\begin{align*}
&
\left| \frac{\partial}{\partial \mu_{k,p}} \left(H[q] - \widetilde{H}[q]\right) \right|
\\
&
\leq
\sum_{k^{\prime}\neq k}
\sqrt{\pi_k \pi_{k^{\prime}}}
\sum_{i=1}^{m}
\left|\gamma_{k^{\prime},ip}^{-1}\right|
\underbrace{
\sqrt{\frac{ |\Gamma_k|}{|\Gamma_{k^{\prime}}|}}
\int_{\mathbb{R}^{m}}\frac{ \left[\Theta(y\mid\mu_{k;k'},\Gamma_{k;k'})\right]_{i} }{\sqrt{(2\pi)^{m}}} 
\exp\left(
\frac{-|y|^{2}-|\Theta(y\mid\mu_{k;k'},\Gamma_{k;k'})|^{2}}{4}
\right)
dy
}
_{\displaystyle\leq \frac{2}{(1-s)^{\frac{m+2}{4}}} \exp
\left(
-\frac{s\alpha_{k,k^{\prime}}^{2}}{4}
\right)} \hspace{30pt}
\\
&
\hspace{20pt}+
\sum_{\ell \neq k}\sqrt{\pi_{k}\pi_{\ell}}
\sum_{i=1}^{m}
\left|\gamma_{k,ip}^{-1}\right|
\underbrace{
\sqrt{\frac{ |\Gamma_{\ell}|}{|\Gamma_{k}|}} 
\int_{\mathbb{R}^{m}}\frac{\left[\Theta(y\mid\mu_{\ell;k},\Gamma_{\ell;k})\right]_{i}}{\sqrt{(2\pi )^{m}}} 
\exp
\left(
\frac{-|y|^{2}-|\Theta(y\mid\mu_{\ell;k},\Gamma_{\ell;k})|^{2}}{4}
\right)
dy}
_{\displaystyle\leq \frac{2}{(1-s)^{\frac{m+2}{4}}} 
\exp
\left(
-\frac{s\alpha_{\ell,k}^{2}}{4}
\right)}\\
&\leq
\frac{2}{(1-s)^{\frac{m+2}{4}}}
\sum_{k^{\prime}\neq k}
\sqrt{\pi_k \pi_{k^{\prime}}}
\Biggl\{
\underbrace{
\left(
\sum_{i=1}^{m}
\left|\gamma_{k^{\prime},ip}^{-1}\right|
\right)
}_{\displaystyle\leq \left\| \Gamma_{k^{\prime}}^{-1} \right\|_{1}}
\exp
\left(
-\frac{s\alpha_{k,k^{\prime}}^{2}}{4}
\right)
+
\underbrace{
\left(
\sum_{i=1}^{m}
\left|\gamma_{k,ip}^{-1}\right|
\right)
}_{\displaystyle\leq \left\| \Gamma_{k}^{-1} \right\|_{1}}
\exp
\left(
-\frac{s\alpha_{k^{\prime}, k}^{2}}{4}
\right)
\Biggr\},
\end{align*}
where the last inequality is given by the same arguments in the proof of Lemma~\ref{part2}.
Similar to Lemma \ref{part2.1}, this evaluation holds when $\alpha_{k,k'}$ and $\alpha_{k',k}$ are replaced with any $\alpha \leq \alpha_{\{k,k'\}}$, for example $\max(\alpha_{k,k'},\alpha_{k',k})$.

\vspace{3mm}

\noindent{\bf (ii)}\ 
The derivatives of \eqref{eq.A.3} with respect to $\gamma_{k,pq}$ are calculated as follows: 
\begin{align}
&\frac{\partial}{\partial \gamma_{k,pq}} \left(H[q] - \widetilde{H}[q]\right)\notag\\
&
=
\pi_{k} \int_{\mathbb{R}^{m}}\frac{1}{\sqrt{(2\pi )^{m}}} \exp\left(-\frac{|y|^{2}}{2}\right)
\frac{\displaystyle\sum_{k^{\prime}\neq k}\frac{\pi_{k^{\prime}} |\Gamma_k|}{\pi_{k}|\Gamma_{k^{\prime}}|} 
\exp\left(
\frac{|y|^{2}-|\Theta(y\mid\mu_{k;k'},\Gamma_{k;k'})|^{2}}{2}
\right)
}
{\displaystyle1+\sum_{k^{\prime}\neq k}\frac{\pi_{k^{\prime}} |\Gamma_k|}{\pi_{k}|\Gamma_{k^{\prime}}|} 
\exp\left(
\frac{|y|^{2}-|\Theta(y\mid\mu_{k;k'},\Gamma_{k;k'})|^{2}}{2}
\right)}\notag
\\
&
\hspace{60pt}
\times
\Biggl\{
\underbrace{
|\Gamma_k|^{-1}
\left(
\frac{\partial}{\partial \gamma_{k,pq}}
|\Gamma_k|
\right)
}_{\displaystyle=|\Gamma_k|^{-1}|\Gamma_{k, pq}|}
+
\frac{\partial}{\partial \gamma_{k,pq}}
\left(
\frac{-|\Theta(y\mid\mu_{k;k'},\Gamma_{k;k'})|^{2}}{2}
\right)
\Biggr\}
dy \notag
\\
&
\hspace{20pt}+
\sum_{\ell \neq k}\pi_{\ell} \int_{\mathbb{R}^{m}}\frac{1}{\sqrt{(2\pi )^{m}}} \exp\left(-\frac{|y|^{2}}{2}\right)
\frac{\displaystyle\frac{\pi_{k} |\Gamma_{\ell}|}{\pi_{\ell}|\Gamma_{k}|} 
\exp
\left(
\frac{|y|^{2}-|\Theta(y\mid\mu_{\ell;k},\Gamma_{\ell;k})|^{2}}{2}
\right)
}
{\displaystyle1+
\sum_{k^{\prime} \neq \ell}
\frac{\pi_{k^{\prime}} |\Gamma_{\ell}|}{\pi_{\ell}|\Gamma_{k^{\prime}}|}
\exp
\left(
\frac{|y|^{2}-|\Theta(y\mid\mu_{\ell;k},\Gamma_{\ell;k})|^{2}}{2}
\right)
}\notag
\\
&
\hspace{80pt}
\times
\Biggl\{
\underbrace{
|\Gamma_{k}|
\frac{\partial}{\partial \gamma_{k,pq}}
|\Gamma_{k}|^{-1}
}_{\displaystyle=|\Gamma_k|^{-1}|\Gamma_{k,pq}|}
+
\frac{\partial}{\partial \gamma_{k,pq}} 
\left(
\frac{-|\Theta(y\mid\mu_{\ell;k},\Gamma_{\ell;k})|^{2}}{2}
\right)
\Biggr\}
dy.\label{derivative-gamma-1}
\end{align}

We calculate
\begin{equation}
\frac{\partial}{\partial \gamma_{k,pq}}
\left(
\frac{-|\Theta(y\mid\mu_{k;k'},\Gamma_{k;k'})|^{2}}{2}
\right)
=
-\sum_{i=1}^{m}\left[\Theta(y\mid\mu_{k;k'},\Gamma_{k;k'})\right]_{i} \gamma_{k^{\prime},ip}^{-1}y_{q},
\label{derivative-gamma-2}\noeqref{derivative-gamma-2}
\end{equation}
and
\begin{equation}
\frac{\partial}{\partial \gamma_{k,pq}}
\left(
\frac{-|\Theta(y\mid\mu_{\ell;k},\Gamma_{\ell;k})|^{2}}{2}
\right)
=
-\sum_{i=1}^{m}\left[\Theta(y\mid\mu_{\ell;k},\Gamma_{\ell;k}) \right]_{i}
\left[
\left(
\frac{\partial}{\partial \gamma_{k,pq}}\Gamma_{k}^{-1}
\right)
\Gamma_{k}
\left(
\Theta(y\mid\mu_{\ell;k},\Gamma_{\ell;k})
\right)
\right]_{i},
\label{derivative-gamma-3}
\end{equation}
where we denote by $[v]_{i}$ the $i$-th component of vector $v$, and $\gamma_{k^{\prime},ip}^{-1}$ and $\gamma_{k,ip}^{-1}$ the $(i,p)$-th component of matrix $\Gamma_{k^{\prime}}^{-1}$ and $\Gamma_{k}^{-1}$, respectively, and $y_q$ is the $q$-th component of vector $y$, and $\frac{\partial}{\partial \gamma_{k,pq}}\Gamma_{k}^{-1}$ is component-wise derivative of matrix $\Gamma_{k}^{-1}$ with respect to $\gamma_{k,pq}$.
We also calculate
$
\left(
\frac{\partial}{\partial \gamma_{k,pq}}\Gamma_{k}^{-1}
\right)
\Gamma_{k}
=
\delta_{k, pq} \Gamma_{k}^{-1}$,
where $\delta_{k, pq}$ is the matrix such that $(p,q)$-th component is one, and other are zero.
We further estimate \eqref{derivative-gamma-3} by
\begin{equation}
\begin{split}
\left| 
\frac{\partial}{\partial \gamma_{k,pq}}
\left(
\frac{-|\Theta(y\mid\mu_{\ell;k},\Gamma_{\ell;k})|^{2}}{2}
\right)
\right|
&
=
\left| 
\left<
\Theta(y\mid\mu_{\ell;k},\Gamma_{\ell;k}),
\delta_{k, pq} \Gamma_{k}^{-1}
\left(
|\Theta(y\mid\mu_{\ell;k},\Gamma_{\ell;k})|
\right)
\right>
\right|
\\
&
\leq
\sum_{i=1}^{m}\sum_{j=1}^{m}
\left|
\left[
\delta_{k, pq} \Gamma_{k}^{-1}
\right]_{ij}
\right|
\left|\left[|\Theta(y\mid\mu_{\ell;k},\Gamma_{\ell;k})| \right]_{i}
\left[|\Theta(y\mid\mu_{\ell;k},\Gamma_{\ell;k})|\right]_{j}\right|,
\end{split}
\label{derivative-gamma-4}
\end{equation}
where $\left[
\delta_{k, pq} \Gamma_{k}^{-1}
\right]_{ij}$ is $(i,j)$-th component of matrix $\delta_{k, pq} \Gamma_{k}^{-1}$.
By using inequality of $\frac{x}{a+x} \leq x^{\frac{1}{2}}$ for $x>1$ ($a\geq 1$), and the arguments \eqref{derivative-gamma-1}--\eqref{derivative-gamma-4}, we have
\begin{align*}
&
\left| \frac{\partial}{\partial \gamma_{k,pq}} \left(H[q] - \widetilde{H}[q]\right) \right|
\\
&
\leq
\sum_{k^{\prime}\neq k}
\sqrt{\pi_k \pi_{k^{\prime}}}
|\Gamma_k|^{-1}|\Gamma_{k, pq}|
\underbrace{
\sqrt{\frac{ |\Gamma_k|}{|\Gamma_{k^{\prime}}|}}
\int_{\mathbb{R}^{m}}\frac{1}{\sqrt{(2\pi)^{m}}} 
\exp\left(
\frac{-|y|^{2}-|\Theta(y\mid\mu_{k;k'},\Gamma_{k;k'})|^{2}}{4}
\right)
dy
}
_{\displaystyle\leq \frac{2}{(1-s)^{\frac{m}{4}}} \exp
\left(
-\frac{s\alpha_{k,k^{\prime}}^{2}}{4}
\right)}
\\
&
\hspace{20pt}+
\sum_{k^{\prime}\neq k}
\sqrt{\pi_k \pi_{k^{\prime}}}
\sum_{i=1}^{m}
\left|\gamma_{k^{\prime},ip}^{-1}\right|
\underbrace{
\sqrt{\frac{ |\Gamma_k|}{|\Gamma_{k^{\prime}}|}}
\int_{\mathbb{R}^{m}}\frac{\left[\Theta(y\mid\mu_{k;k'},\Gamma_{k;k'})\right]_{i} }{\sqrt{(2\pi)^{m}}} 
\exp\left(
\frac{-|y|^{2}-|\Theta(y\mid\mu_{k;k'},\Gamma_{k;k'})|^{2}}{4}
\right)
dy
}
_{\displaystyle\leq \frac{2}{(1-s)^{\frac{m+2}{4}}} \exp
\left(
-\frac{s\alpha_{k,k^{\prime}}^{2}}{4}
\right)}
\\
&
\hspace{20pt}+
\sum_{\ell \neq k}\sqrt{\pi_{k}\pi_{\ell}}
\sum_{i=1}^{m}
|\Gamma_k|^{-1}|\Gamma_{k, pq}|
\underbrace{
\sqrt{\frac{ |\Gamma_{\ell}|}{|\Gamma_{k}|}} 
\int_{\mathbb{R}^{m}}\frac{1}{\sqrt{(2\pi )^{m}}} 
\exp
\left(
\frac{-|y|^{2}-|\Theta(y\mid\mu_{\ell;k},\Gamma_{\ell;k})|^{2}}{4}
\right)
dy}
_{\displaystyle\leq \frac{2}{(1-s)^{\frac{m}{4}}} 
\exp
\left(
-\frac{s\alpha_{\ell,k}^{2}}{4}
\right)}
\\
&
\hspace{20pt}+
\sum_{\ell \neq k}\sqrt{\pi_{k}\pi_{\ell}}
\sum_{i=1}^{m}\sum_{j=1}^{m}
\left|\left[
\delta_{k, pq} \Gamma_{k}^{-1}
\right]_{ij}\right|\\
& \hspace{40pt}\times 
\underbrace{
\sqrt{\frac{ |\Gamma_{\ell}|}{|\Gamma_{k}|}} 
\int_{\mathbb{R}^{m}}\frac{\left|\left[\Theta(y\mid\mu_{\ell;k},\Gamma_{\ell;k})\right]_{i}
\left[|\Theta(y\mid\mu_{\ell;k},\Gamma_{\ell;k})| \right]_{j}\right|}{\sqrt{(2\pi )^{m}}} 
\exp
\left(
\frac{-|y|^{2}-|\Theta(y\mid\mu_{\ell;k},\Gamma_{\ell;k})|^{2}}{4}
\right)
dy
}
_{\displaystyle\leq \frac{6}{(1-s)^{\frac{m+4}{4}}} 
\exp
\left(
-\frac{s\alpha_{\ell,k}^{2}}{4}
\right)}
\\
&
\leq
\frac{6}{(1-s)^{\frac{m+4}{4}}}
\sum_{k^{\prime}\neq k}
\sqrt{\pi_k \pi_{k^{\prime}}}
\Biggl[
\Biggl\{
|\Gamma_k|^{-1}|\Gamma_{k, pq}|
+
\underbrace{
\left(
\sum_{i=1}^{m}
\left|\gamma_{k^{\prime},ip}^{-1}\right|
\right)
}_{\displaystyle\leq \left\| \Gamma_{k^{\prime}}^{-1} \right\|_{1}}
\Biggr\}
\exp
\left(
-\frac{s\alpha_{k,k^{\prime}}^{2}}{4}
\right)
\hspace{110pt}
\\
&
\hspace{80pt}
+
\Biggl\{
|\Gamma_k|^{-1}|\Gamma_{k, pq}|
+
\underbrace{
\left(
\sum_{i=1}^{m}\sum_{j=1}^{m}
\left|\left[
\delta_{k, pq} \Gamma_{k}^{-1}
\right]_{ij}\right|
\right)
}_{\displaystyle\leq \left\| \Gamma_{k}^{-1} \right\|_{1}}
\Biggr\}
\exp
\left(
-\frac{s\alpha_{k^{\prime}, k}^{2}}{4}
\right)
\Biggr],
\end{align*}
where the last inequality is given by the same arguments in the proof of Theorem~\ref{part2}.

\vspace{3mm}

\noindent{\bf (iii)}\ 
The derivatives of \eqref{eq.A.3} with respect to $\pi_{k}$ are calculated as follows:
\begin{equation*}
\begin{split}
&\frac{\partial}{\partial \pi_{k}} \left(H[q] - \widetilde{H}[q]\right)\\
&
= \int_{\mathbb{R}^{m}}\frac{1}{\sqrt{(2\pi )^{m}}} \exp\left(-\frac{|y|^{2}}{2}\right)
\log 
\left(
1+\sum_{k^{\prime}\neq k}\frac{\pi_{k^{\prime}} |\Gamma_k|}{\pi_{k}|\Gamma_{k^{\prime}}|} 
\exp\left(
\frac{|y|^{2}-|\Theta(y\mid\mu_{k;k'},\Gamma_{k;k'})|^{2}}{2}
\right)
\right)
dy
\\
&
\hspace{20pt}+
\pi_{k} \int_{\mathbb{R}^{m}}\frac{1}{\sqrt{(2\pi )^{m}}} \exp\left(-\frac{|y|^{2}}{2}\right)
\frac{\displaystyle
\sum_{k^{\prime} \neq k}
\frac{-\pi_{k^{\prime}} |\Gamma_{k}|}{\pi_{k}^{2}|\Gamma_{k^{\prime}}|}
\exp
\left(
\frac{|y|^{2}-|\Theta(y\mid\mu_{k;k'},\Gamma_{k;k'})|^{2}}{2}
\right)
}
{\displaystyle1+
\sum_{k^{\prime} \neq k}
\frac{\pi_{k^{\prime}} |\Gamma_{k}|}{\pi_{k}|\Gamma_{k^{\prime}}|}
\exp
\left(
\frac{|y|^{2}-|\Theta(y\mid\mu_{k;k'},\Gamma_{k;k'})|^{2}}{2}
\right)
}
dy
\\
&
\hspace{20pt}+
\sum_{\ell \neq k}\pi_{\ell} \int_{\mathbb{R}^{m}}\frac{1}{\sqrt{(2\pi )^{m}}} \exp\left(-\frac{|y|^{2}}{2}\right)
\frac{\displaystyle\frac{ |\Gamma_{\ell}|}{\pi_{\ell}|\Gamma_{k}|} 
\exp
\left(
\frac{|y|^{2}-|\Theta(y\mid\mu_{\ell;k},\Gamma_{\ell;k})|^{2}}{2}
\right)
}
{\displaystyle1+
\sum_{k^{\prime} \neq \ell}
\frac{\pi_{k^{\prime}} |\Gamma_{\ell}|}{\pi_{\ell}|\Gamma_{k^{\prime}}|}
\exp
\left(
\frac{|y|^{2}-|\Theta(y\mid\mu_{\ell;k},\Gamma_{\ell;k})|^{2}}{2}
\right)
}
dy.
\\
\end{split}
\end{equation*}
Using inequalities of $\log(1+ x) \leq \sqrt{x}$ for $x \geq 0$ and  $\frac{x}{a+x} \leq x^{\frac{1}{2}}$ for $x>1$ when $a\geq 1$, we estimate
\begin{align*}
\left| \frac{\partial}{\partial \pi_{k}} \left(H[q] - \widetilde{H}[q]\right) \right|
&
\leq
\sum_{k^{\prime}\neq k}
\sqrt{\frac{\pi_{k^{\prime}}}{\pi_k}}
\underbrace{
\sqrt{\frac{ |\Gamma_k|}{|\Gamma_{k^{\prime}}|}}
\int_{\mathbb{R}^{m}}\frac{1}{\sqrt{(2\pi)^{m}}} 
\exp\left(
\frac{-|y|^{2}-|\Theta(y\mid\mu_{k;k'},\Gamma_{k;k'})|^{2}}{4}
\right)
dy
}
_{\displaystyle\leq \frac{2}{(1-s)^{\frac{m}{4}}} \exp
\left(
-\frac{s\alpha_{k,k^{\prime}}^{2}}{4}
\right)}
\\
&
\hspace{20pt}+
\sum_{k^{\prime}\neq k}
\sqrt{\frac{\pi_{k^{\prime}}}{\pi_k}}
\underbrace{
\sqrt{\frac{ |\Gamma_k|}{|\Gamma_{k^{\prime}}|}}
\int_{\mathbb{R}^{m}}\frac{1}{\sqrt{(2\pi)^{m}}} 
\exp\left(
\frac{-|y|^{2}-|\Theta(y\mid\mu_{k;k'},\Gamma_{k;k'})|^{2}}{4}
\right)
dy
}
_{\displaystyle\leq \frac{2}{(1-s)^{\frac{m}{4}}} \exp
\left(
-\frac{s\alpha_{k,k^{\prime}}^{2}}{4}
\right)}
\\
&
\hspace{20pt}+
\sum_{\ell \neq k}
\sqrt{ \frac{\pi_{\ell}}{\pi_{k}}}
\underbrace{
\sqrt{\frac{ |\Gamma_{\ell}|}{|\Gamma_{k}|}} 
\int_{\mathbb{R}^{m}}
\frac{1}{\sqrt{(2\pi )^{m}}} 
\exp
\left(
\frac{-|y|^{2}-|\Theta(y\mid\mu_{\ell;k},\Gamma_{\ell;k})|^{2}}{4}
\right)
dy}
_{\displaystyle\leq \frac{2}{(1-s)^{\frac{m}{4}}} 
\exp
\left(
-\frac{s\alpha_{\ell,k}^{2}}{4}
\right)}\\
&\leq
\frac{4}{(1-s)^{\frac{m}{4}}}
\sum_{k^{\prime}\neq k}
\sqrt{\frac{\pi_{k^{\prime}}}{\pi_k}}
\Biggl\{
\exp
\left(
-\frac{s\alpha_{k,k^{\prime}}^{2}}{4}
\right)
+
\exp
\left(
-\frac{s\alpha_{k^{\prime}, k}^{2}}{4}
\right)
\Biggr\},
\end{align*}
where the last inequality is given by the same arguments in the proof of Theorem~\ref{part2}.
The proof of Theorem~\ref{derivative-entropy} is finished.
\end{proof}
\vspace{0mm}
\subsection{Proof of Proposition \ref{explicit form 1}}\label{Proof of Lemma 4.4}
\begin{lem}
Let $m \geq K \geq 2$.
Then
\begin{equation}
H[q]  = \widetilde{H}[q] - \sum_{k=1}^{K}\frac{\pi_k}{(2\pi)^{\frac{K-1}{2}}}
\int_{\mathbb{R}^{K-1}} \exp\left(-\frac{|v|^2}{2}\right)
\log\left(1 + \sum_{k^{\prime}\neq k}  \frac{\pi_{k^{\prime}}}{\pi_k}\exp\left(\frac{|v|^2 - \left|v-u_{k^{\prime},k}\right|^{2}}{2}\right) \right)dv,
\label{explicit_form_app}
\end{equation}
where $u_{k^{\prime},k}\coloneqq[R_{k} \Sigma^{-1/2} (\mu_{k^{\prime}} - \mu_{k})]_{1:K-1} \in \mathbb{R}^{K-1}$ and $R_{k} \in \mathbb{R}^{m \times m}$ is some rotation matrix such that
\begin{equation}
\label{assumption_rotation_app}
R_{k}\Sigma^{-\frac{1}{2}}(\mu_{k^{\prime}} - \mu_k ) \in \mathrm{span}\{e_1, \cdots,e_{K-1}\}, \,\,\, k^{\prime} \in [K].
\end{equation}
Here, $\{e_i\}_{i=1}^{K-1}$ is the standard basis in $\mathbb{R}^{K-1}$, and
$u_{1:K-1}\coloneqq(u_1,\ldots,u_{K-1})^{T}\in \mathbb{R}^{K-1}$ for $u=(u_1,\ldots,u_m)^{T}\in \mathbb{R}^{m}$.
\end{lem}
\begin{proof}
First, we observe that
\begin{equation*}
H[q] = \widetilde{H}[q] - 
\underbrace{\sum_{k=1}^{K}\pi_{k} \int_{\mathbb{R}^{m}}  \mathcal{N}(x| \mu_{k}, \Sigma_{k}) \left\{\log \left(\sum_{k^{\prime}=1}^{K}\pi_{k^{\prime}} \mathcal{N}(x| \mu_{k^{\prime}}, \Sigma_{k}) \right) - \log \left(\pi_{k} \mathcal{N}(x| \mu_{k}, \Sigma_{k}) \right)  \right\}dx}_{\displaystyle \eqqcolon(\clubsuit)}.
\end{equation*}
Using \eqref{computation1} with the assumption \eqref{assumption_common}, we write 
\begin{equation*}
\begin{split}
(\clubsuit)
& =
\sum_{k=1}^{K}\pi_{k} \int_{\mathbb{R}^{m}}\frac{1}{\sqrt{(2\pi )^{m}}} \exp\left(-\frac{|y|^{2}}{2}\right)
\log \left(1+\sum_{k^{\prime} \neq  k}\frac{\pi_{k^{\prime}}}{\pi_{k}} \exp\left(\frac{|y|^{2}-\left|\left(y-\Sigma^{-\frac{1}{2}}(\mu_{k^{\prime}}-\mu_{k})\right)\right|^{2}}{2}\right) \right)dy. 
\end{split}
\end{equation*}
We choose the rotation matrix $R_{k} \in \mathbb{R}^{m \times m}$ satisfying \eqref{assumption_rotation_app} for each $k \in [K]$.
Making the change of variables as $z=R_{k}y$, we write
\makeatletter
\begin{equation*}
\begin{split}
(\clubsuit)
& =
\sum_{k=1}^{K}\pi_{k} \int_{\mathbb{R}^{m}}\frac{1}{\sqrt{(2\pi )^{m}}} \exp\left(-\frac{|z|^{2}}{2}\right)
\log \bBigg@{3.5}(1+\sum_{k^{\prime}\neq k}\frac{\pi_{k^{\prime}}}{\pi_{k}}
\underbrace{\exp\left(\frac{|z|^{2}-\left|\left(R^{T}z-\Sigma^{-\frac{1}{2}}(\mu_{k^{\prime}}-\mu_{k})\right)\right|^{2}}{2}\right)}
_
{\displaystyle=\exp\left(\frac{|z|^{2}-\left|z-u_{k^{\prime},k}\right|^{2}}{2}\right)} \bBigg@{3.5})dy, 
\end{split}
\end{equation*}
where $u_{k^{\prime},k}=[R_{k} \Sigma^{-1/2} (\mu_{k^{\prime}} - \mu_{k})]_{1:K-1} \in \mathbb{R}^{K-1}$, that is, 
\[
R_{k} \Sigma^{-\frac{1}{2}} (\mu_{k^{\prime}} - \mu_{k})
=
(u_{k^{\prime},k}, 0,\cdots,0)^{T}.
\]
We change the variables as 
\[
z_1 = v_1, \,\, \ldots, \,\, z_{K-1}=v_{K-1}, \,\, 
z_{K} = r\cos\theta_{K}, \,\,
z_{K+1} = r\sin\theta_{K}\cos\theta_{K+1}, \,\, \ldots 
\]
\[
z_{m-1} = r\sin\theta_{K} \cdots \sin\theta_{m-2}\cos\theta_{m-1}, \,\,
z_{m} = r\sin\theta_{K} \cdots \sin\theta_{m-2}\sin\theta_{m-1},
\]
where $-\infty<v_{i}<\infty$ ($i=1,\ldots,K-1$), $r>0$, $0<\theta_j<\pi$ ($j=K,\ldots,m-2$), and $0<\theta_{m-1}<2\pi$. 
Because 
\[
|z|^{2}=|v|^{2}+r^{2}, \,\,\, |z-u_{k^{\prime},k}|^{2}=|v-u_{k^{\prime},k}|^{2}+r^{2},
\]
\[
dz_1 \cdots dz_m = r^{m-K} \prod_{i=K}^{m-1} (\sin\theta_i)^{m-i-1}dv_1 \cdots dv_{K-1}\,dr\,d\theta_{K}\cdots d\theta_{m-1},
\]
we obtain
\begin{equation}
\begin{split}
(\clubsuit)&=
\sum_{k=1}^{K}\frac{\pi_{k}}{\sqrt{(2\pi )^{m}}}
\int_{0}^{2\pi}d\theta_{m-1}
\int_{0}^{\pi}d\theta_{m-2} \,\cdots
\int_{0}^{\pi}d\theta_{K}
\int_{0}^{\infty}dr
\int_{-\infty}^{\infty}dv_{K-1} \,\cdots
\int_{-\infty}^{\infty}dv_{1} \\
& \qquad \times
\exp\left(-\frac{|v|^{2}+r^{2}}{2}\right) 
\log \Biggl(1+\sum_{k^{\prime}\neq k}\frac{\pi_{k^{\prime}}}{\pi_{k}} \exp\left(\frac{|v|^{2}-\left|v-u_{k^{\prime},k}\right|^{2}}{2}\right) \Biggr)
r^{m-K} \prod_{i=K}^{m-1} (\sin\theta_i)^{m-i-1}.
\label{after_change_Cylinder}
\end{split}
\end{equation}
By the equality
\[
\begin{split}
&\frac{1}{\sqrt{(2\pi)^{m}}}
\int_{0}^{2\pi}d\theta_{m-1}
\int_{0}^{\pi}d\theta_{m-2} \,\cdots
\int_{0}^{\pi}d\theta_{K}
\int_{0}^{\infty}dr \
\exp\left(-\frac{r^2}{2} \right) r^{m-K} \prod_{i=K}^{m-1} (\sin\theta_i)^{m-i-1}
= \frac{1}{(2\pi)^{\frac{K-1}{2}}},
\end{split}
\]
we conclude \eqref{explicit_form_app} from \eqref{after_change_Cylinder}. The proof of Proposition \ref{explicit form 1} is finished.
\end{proof}

\section{Details of Section \ref{sec:experiment-error}}
\label{app:experiment-error}
We give a detailed explanation for the relative error experiment in Section~\ref{sec:experiment-error}.
We restricted the setting of the experiment to the case for the coincident covariance matrices \eqref{assumption_common} and the number of mixture components $K=2$.
Furthermore, we assumed that $\Sigma = I$, $\mu_1=0$, and $\mu_2 \sim \mathcal{N}(0, (2c)^2I)$.
In this setting, we varied the dimension $m$ of Gaussian distributions from 1 to 500 for certain parameters $(c, \pi_k)$, where we sampled $\mu_2$ 10 times for each dimension $m$.
As formulas for the entropy approximation, we employed $\widetilde{H}_{\rm ours}[q]$, $\widetilde{H}_{{\rm Huber}(R)}[q]$, and $\widetilde{H}_{\rm Bonilla}[q]$ for the method of ours, \citet{huber2008entropy}, and \citet{bonilla2019generic}, respectively, as follows:
\begin{align}
\widetilde{H}_{\rm ours}[q] 
&= \frac{m}{2} +  \frac{m}{2} \log 2 \pi + \frac{1}{2} \sum_{k=1}^{K} \pi_{k} \log |\Sigma_{k}| 
- \sum_{k=1}^{K} \pi_{k} \log \pi_{k},
\label{eq:entropy-approximation-ours} \\
\begin{split}
\widetilde{H}_{{\rm Huber}(R)}[q]
&= -\sum_{k=1}^{K}\pi_{k} \int \mathcal{N}(w|\mu_k, \Sigma_{k})\sum_{i=0}^{R}\frac{1}{i !}\left((w-\mu_k)\odot \nabla_{\widetilde{w}} \right)^{i}\log\left(\sum_{k^{\prime}=1}^{K} \pi_{k^{\prime}} \mathcal{N}(\widetilde{w}|\mu_{k^{\prime}}, \Sigma_{k^{\prime}}) \right)\Biggl|_{\widetilde{w}=\mu_k} dw, 
\end{split}
\label{eq:entropy-approximation-huber} \\
\widetilde{H}_{\rm Bonilla}[q]
&= -\sum_{k=1}^{K}\pi_{k} \log
\left( \sum_{k^{\prime}=1}^{K} \mathcal{N}(\mu_k|\mu_{k^{\prime}}, \Sigma_{k}+\Sigma_{k^{\prime}}) \right),
\label{eq:entropy-approximation-bonilla}
\end{align}
where \eqref{eq:entropy-approximation-ours} is the same as described in Section~\ref{sec:Entropy-approximtion}, \eqref{eq:entropy-approximation-huber} is based on the Taylor expansion \citep[(4)]{huber2008entropy}, and \eqref{eq:entropy-approximation-bonilla} is based on the lower bound analysis \citep[(14)]{bonilla2019generic}.
In the case for the coincident and diagonal covariance matrices $\Sigma_k=\mathrm{diag}(\sigma_{1}^{2},\ldots,\sigma_{m}^{2})$, \eqref{eq:entropy-approximation-huber} for $R=0$ or $2$ can be analytically computed as
\begin{align*}
\widetilde{H}_{\rm Huber(0)}[q]
&= -\sum_{k=1}^{K}\pi_{k}  
\log\left(\sum_{k^{\prime}=1}^{K} \pi_{k^{\prime}} \mathcal{N}(\mu_k |\mu_{k^{\prime}}, \Sigma_{k^{\prime}}) \right),\\
\widetilde{H}_{\rm Huber(2)}[q]
&= -\sum_{k=1}^{K}\pi_{k}  
\left[
\log\left(\sum_{k^{\prime}=1}^{K} \pi_{k^{\prime}} \mathcal{N}(\mu_k |\mu_{k^{\prime}}, \Sigma_{k^{\prime}}) \right)
+\frac{1}{2}
\sum_{i=1}^{m} \sigma^{2}_{i}C_{k,i}
\right],
\end{align*}
where
\begin{align*}
C_{k,i}
&\coloneqq\frac{g_{0,k}g_{2,k,i}-g_{1,k,i}^{2}}{g_{0,k}^{2}}, \quad
g_{0,k}
\coloneqq\sum_{k^{\prime}=1}^{K}\pi_{k^{\prime}} \mathcal{N}(\mu_k |\mu_{k^{\prime}}, \Sigma_{k^{\prime}}), \\
g_{1,k,i}
&\coloneqq\sum_{k^{\prime}=1}^{K}\pi_{k^{\prime}} 
\frac{\mu_{k,i}-\mu_{k^{\prime},i}}{\sigma_{i}^{2}}
\mathcal{N}(\mu_k |\mu_{k^{\prime}}, \Sigma_{k^{\prime}}), \\
g_{2,k,i}
&\coloneqq\sum_{k^{\prime}=1}^{K}\pi_{k^{\prime}} 
\left[
    \left(
    \frac{\mu_{k,i}-\mu_{k^{\prime},i}}{\sigma_{i}^{2}}
    \right)^2
    -\frac{1}{\sigma_{i}^{2}}
\right]
\mathcal{N}(\mu_k |\mu_{k^{\prime}}, \Sigma_{k^{\prime}}).
\end{align*}

In the following, we show the tractable formula of the true entropy, which is used in the experiment in Section~\ref{sec:experiment-error}.
In the case for the coincident covariance matrices and $K=2$, we can reduce the integral in \eqref{explicit_form} to a one-dimensional Gaussian integral as follows:
\begin{align*}
H[q]
=\widetilde{H}[q] -\sum_{k\neq k\prime}\frac{\pi_k}{\sqrt{2\pi}}
\int_{\mathbb{R}}
\exp\left(-\frac{|v|^2}{2}\right)
\log\left(1 + \frac{\pi_{k^{\prime}}}{\pi_k}\exp\left(\frac{|v|^2 - \left|v-u_{k^{\prime},k}\right|^{2}}{2}\right) \right)dv.
\end{align*}
Furthermore, we can choose the rotation matrix $R_k$ in Proposition~\ref{explicit form 1} such that
\begin{align*}
u_{k^{\prime},k}=\left[R_{k}\Sigma^{-\frac{1}{2}} (\mu_{k^{\prime}} - \mu_{k}) \right]_{1:1} \geq 0,
\end{align*}
that is,
\begin{align*}
u_{k^{\prime},k}=\left|u_{k^{\prime},k}\right|=\left|\Sigma^{-\frac{1}{2}}(\mu_{k^{\prime}} - \mu_{k})\right|= 2 |a|, \quad
a\coloneqq\frac{\Sigma^{-\frac{1}{2}}(\mu_1 - \mu_2)}{2}.
\end{align*}
Hence, we have
\begin{align*}
|v|^2 - |v-u_{k^{\prime},k}|^{2} = -u_{k^{\prime},k}^{2} + 2 vu_{k^{\prime},k} = -4|a|^2 + 4v|a|.
\end{align*}
Therefore, by making the change of variables as $v = \sqrt{2}\, t$, we conclude that
\begin{align}
H[q] 
&= \widetilde{H}[q] - \sum_{k=1}^{2}\frac{\pi_k}{\sqrt{2\pi}}
\int_{\mathbb{R}}
\exp\left(-\frac{|v|^2}{2}\right)
\log\left(1 + \frac{\pi_{k^{\prime}}}{\pi_k}
\exp
\left(
-2|a|^2 +2  v|a|
\right) \right)dv \nonumber \\ 
&= \widetilde{H}[q]
-
\sum_{k=1}^{2}\frac{\pi_k}{\sqrt{\pi}}
\int_{\mathbb{R}}
\exp(-t^2)
\log\left(1 + \frac{\pi_{k^{\prime}}}{\pi_k}
\exp
\left(
-2|a|^2 +2\sqrt{2}|a|t
\right) \right)dt. \label{eq:entropy-k2}
\end{align}
Note that the integration of \eqref{eq:entropy-k2} can be efficiently executed using the Gauss-Hermite quadrature.

\section{Application example: Variational inference to BNNs with Gaussian mixtures}\label{Variational inference for Bayesian neural networks}
\subsection{Overview: 
Variational inference with Gaussian mixtures}

Let $f(\,\cdot \ ;w)$ be the base model that is a function parameterized by weights $w \in \mathbb{R}^{m}$, e.g., the neural network.
Let $p(w)$ and $p(y|f(x;w))$ be the prior distribution of the weights and the likelihood of the model, respectively. 
For supervised learning, let $D=\{(x_{n}, y_{n})\}_{n=1}^{N}$ be a dataset where $x_{n}\in\mathbb{R}^{d_x}$ and $y_{n}\in\mathbb{R}^{d_y}$ are the input and output, respectively, and the input-output pair $(x_n, y_n)$ is independently identically distributed.
The Bayesian posterior distribution $p(w|D)$ is formulated as
\[
p(w|D) \propto p(w) \prod_{n=1}^{N} p(y_n|f(x_n;w)).
\]

The goal of variational inference is to minimize the Kullback-Leibler (KL) divergence between a variational family $q_{\theta}(w)$ and posterior distribution $p(w|D)$ given by
\[
D_{\text{KL}}(q_{\theta}(w) \,||\, p(w|D) )\coloneqq-\int q_{\theta}(w)\log \left(\frac{p(w|D)}{q_{\theta}(w)}\right)dw,
\]
which is equivalent to maximizing the evidence lower bound (ELBO) given by
\begin{align}
\label{eq:elbo}
\mathcal{L}(\theta)\coloneqq L(\theta) + \int q_{\theta}(w) \log(p(w))\,dw + H[q_{\theta}],
\end{align}
(see, e.g., \citet{barber1998ensemble, bishop2006PRML, hinton1993keeping}).
The first term of \eqref{eq:elbo} is the expected log-likelihood given by
\[
L(\theta)\coloneqq\sum_{n=1}^{N}E_{q_{\theta}(w)}[\log p(y_n|f(x_n;w))],
\]
the second term is the cross-entropy between a variational family $q(w)$ and prior distribution $p(w)$, and the third term is the entropy of $q(w)$ given by
\begin{align*}
H[q_{\theta}] \coloneqq -\int q_{\theta}(w) \log( q_{\theta}(w))\,dw.
\end{align*}
Here, we choose a unimodal Gaussian distribution as a prior, that is, $p(w) = \mathcal{N}(w| \mu_{0}, \Sigma_{0})$, and
we choose a Gaussian mixture distribution as a variational family, that is,
\[
q_{\theta}(w)= \sum_{k=1}^{K}\pi_{k} \mathcal{N}(w|\mu_{k}, \Sigma_{k}), \ \ \theta = (\pi_{k}, \mu_{k}, \Sigma_{k})_{k=1}^{K},
\]
where $K \in \mathbb{N}$ is the number of mixture components, and $\pi_k \in (0,1]$ are mixing coefficients constrained by $\sum_{k=1}^{K}\pi_k=1$.
Here, $\mathcal{N}(w| \mu_{k}, \Sigma_{k})$ is the Gaussian distribution with a mean $\mu_k \in \mathbb{R}^{m}$ and covariance matrix $\Sigma_k \in \mathbb{R}^{m \times m}$, that is,
\[
\mathcal{N}(w|\mu_{k}, \Sigma_{k}) = \frac{1}{\sqrt{(2\pi)^{m} |\Sigma_{k}|}}\exp\left(-\frac{1}{2}\left\|w-\mu_k\right\|_{\Sigma_k}^{2}\right),
\]
where $|\Sigma_{k}|$ is the determinant of matrix $\Sigma_{k}$, and $\|x\|_{\Sigma}^{2}\coloneqq x \cdot( \Sigma^{-1}x)$ for a vector $x \in \mathbb{R}^{m}$ and a positive definite matrix $\Sigma \in \mathbb{R}^{m \times m}$.
\par
In the following, we investigate the ingredients in \eqref{eq:elbo}.
The expected log-likelihood $L(\theta)$ is analytically intractable due to the nonlinearity of the function $f(x ;w)$.
To overcome this difficulty, we follow the stochastic gradient variational Bayes (SGVB) method \citep{kingma2013auto, kingma2015variational, rezende2014stochastic}, which employs the reparametric trick and minibatch-based Monte Carlo sampling.
Let $S \subset D$ be a minibatch set with minibatch size $M$.
By reparameterizing weights as $w=\Sigma_{k}^{1/2}\varepsilon + \mu_{k}$, we can rewrite the expected log-likelihood $L(\theta)$ as
\[
\begin{split}
L(\theta) &= 
\sum_{n=1}^{N} \sum_{k=1}^{K}\pi_{k} \int \mathcal{N}(w|\mu_k, \Sigma_k) \log p(y_n|f(x_n;w))\, dw \\
&=
\sum_{n=1}^{N} \sum_{k=1}^{K}\pi_{k} \int \mathcal{N}(\varepsilon|0, I)  \log p(y_n|f(x_{n};\Sigma_{k}^{\frac{1}{2}}\varepsilon + \mu_{k}))\, d\varepsilon. 
\end{split}
\]
By minibatch-based Monte Carlo sampling, we obtain the following unbiased estimator $\widehat{L}^{\rm {SGVB}}(\theta)$ of the expected log-likelihood $L(\theta)$ as
\begin{equation}
\begin{split}
L(\theta) & \approx \widehat{L}^{\rm {SGVB}}(\theta)
\\
& \coloneqq \sum_{k=1}^{K}\pi_{k}\frac{N}{M}\sum_{i \in S} \log p(y_i|f(x_{i};\Sigma_{k}^{\frac{1}{2}}\varepsilon + \mu_{k})),
\label{expected log like}
\end{split}
\end{equation}
where we employ noise sampling $\varepsilon \sim \mathcal{N}(0, I)$ once per mini-batch sampling (\citet{kingma2015variational, titsias2014doubly}).
On the other hand, the cross-entropy between a Gaussian mixture and unimodal Gaussian distribution can be analytically computed as
\begin{equation}
\begin{split}
\int q_\theta(w) \log(p(w)) \,dw 
= -\sum_{k=1}^{K}\frac{\pi_{k}}{2} \Biggl\{ m\log 2 \pi + \log |\Sigma_{0}|
+ \mathrm{Tr}(\Sigma_{0}^{-1}\Sigma_{k}) + \left\| \mu_{k}-\mu_{0} \right\|_{\Sigma_{0}}^{2} \Biggr\}.
\label{cross entropy}
\end{split}
\end{equation}
\par
For the entropy term $H[q_{\theta}]$ we employ the approximate entropy \eqref{def:approx_of_entropy}, that is, 
\begin{equation}
\begin{split}
H[q] &\approx \widetilde{H}[q_{\theta}]
\\
&=\, -\sum_{k=1}^{K}\pi_{k} \int \mathcal{N}(w|\mu_k, \Sigma_{k}) \log\left(\pi_{k} \mathcal{N}(w|\mu_k, \Sigma_{k}) \right)dw \\
& = \, \frac{m}{2} +  \frac{m}{2} \log 2 \pi + \frac{1}{2} \sum_{k=1}^{K} \pi_{k} \log |\Sigma_{k}| 
- \sum_{k=1}^{K} \pi_{k} \log \pi_{k}. 
\end{split}
\label{approx_of_entropy-app}
\end{equation}
\par
In summary, we approximate the ELBO $\mathcal{L}(\theta)$ using \eqref{expected log like}, \eqref{cross entropy}, and \eqref{approx_of_entropy-app} as
\begin{equation}
\begin{split}
\mathcal{L}(\theta) &\approx \widehat{\mathcal{L}}(\theta)
\\
&\coloneqq\, \widehat{L}^{\rm {SGVB}}(\theta) + \int q_{\theta}(w) \log(p(w)) \,dw + \widetilde{H}[q_\theta] \\
&=\, 
\sum_{k=1}^{K}\pi_{k}\left(\widehat{\mathcal{L}}(\mu_k, \Sigma_k) - \log\pi_{k}\right),
\label{approx ELBO}
\end{split}
\end{equation}
where $\widehat{\mathcal{L}}(\mu_k, \Sigma_k)$ are ELBOs of the unimodal Gaussian distributions $\mathcal{N}(w|\mu_{k}, \Sigma_{k})$ given by
\[
\begin{split}
\widehat{\mathcal{L}}(\mu_k, \Sigma_k) 
\coloneqq &\frac{N}{M}\sum_{i \in S} \log p(y_i|f(x_{i};\Sigma_{k}^{\frac{1}{2}}\varepsilon_{S} + \mu_{k}))
- \frac{1}{2} \Biggl\{ m\log 2 \pi + \log |\Sigma_{0}|  
+ \mathrm{Tr}(\Sigma_{0}^{-1}\Sigma_{k}) + \left\| \mu_{k}-\mu_{0} \right\|_{\Sigma_{0}}^{2} \Biggr\} 
\\
&\quad + \frac{m}{2}(1 + \log 2\pi) + \frac{1}{2} \log|\Sigma_{k}|.
\end{split}
\]

\subsection{Experiment of BNN with Gaussian mixtures on toy task}
\label{sec:experiment-bnn}

We employed the approximate entropy \eqref{def:approx_of_entropy}
for variational inference to a BNN 
whose posterior was modeled by the Gaussian mixture, which we call {\it {BNN-GM}} in the following. 
We conducted the toy $1$~D regression task~\citep{he2020bayesian} to observe the uncertainty estimation capability of the BNN-GM.
In particular, we observed that the BNN-GM could capture larger uncertainty than the deep ensemble~\citep{lakshminarayanan2016simple}.
The task was to learn a curve $y = x \sin(x)$ from a training dataset that consisted of $20$ points sampled from a noised curve $y = x \sin(x) + \varepsilon$, $\varepsilon \sim \mathcal{N}(0, 0.1^2)$.
Refer to Appendix~\ref{app:experiment-bnn} for the detail of implementations.
We compared the BNN-GM with the deep ensemble of DNNs, the BNN with the single unimodal Gaussian, and the deep ensemble of BNNs with the single unimodal Gaussian, see Figure~\ref{fig:bnngm}.

From the result in Figure~\ref{fig:bnngm}, we can observe the following.
First, every method can represent uncertainty on the area where train data do not exist.
However, the BNN with a single unimodal Gaussian can represent smaller uncertainty than other methods (see around $x=0$).
Second, as increasing the number of components, the BNN-GM can represent larger uncertainty than the deep ensemble of DNNs or BNNs.
Therefore, there is a qualitative difference in uncertainty estimation between BNN-GMs and deep ensembles.
Finally, the BNN-GM of 10 components has weak learners with small mixture coefficients.
We suppose that this phenomenon is caused by the entropy regularization for the Gaussian mixture.
Note that we do not claim the superiority of BNN-GMs to deep ensembles.

\begin{figure*}[t] 
\centering
\begin{minipage}[b]{0.32\linewidth}
\includegraphics[width=\linewidth]{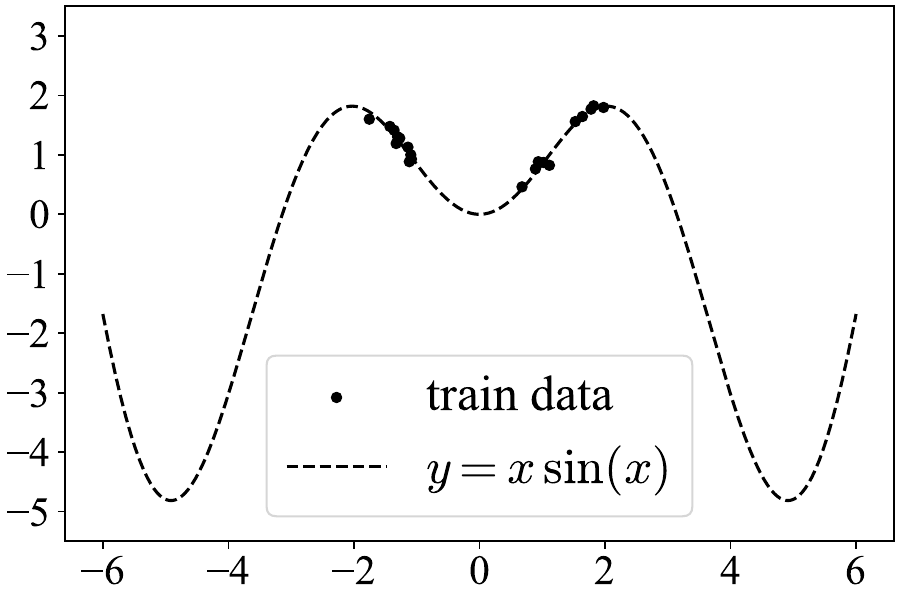}
\begin{center} \vspace{-5pt}\small (a) \ Dataset\end{center}\label{fig:dataset}
\end{minipage}
\begin{minipage}[b]{0.32\linewidth}
\includegraphics[width=\linewidth]{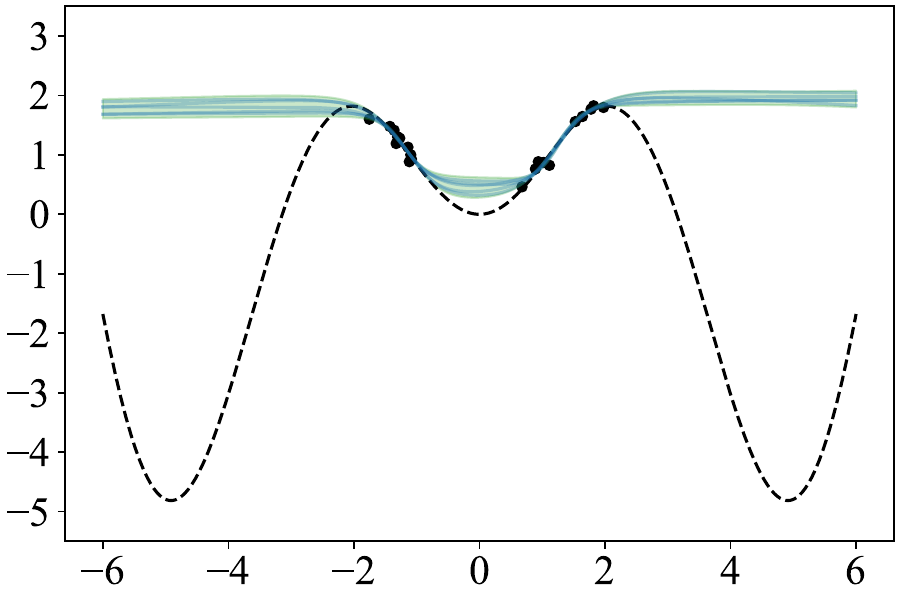}
\begin{center}\vspace{-5pt}\small (b) \ Deep ensemble of 5 DNNs\end{center}\label{fig:5DNN}
\end{minipage}
\begin{minipage}[b]{0.32\linewidth}
\includegraphics[width=\linewidth]{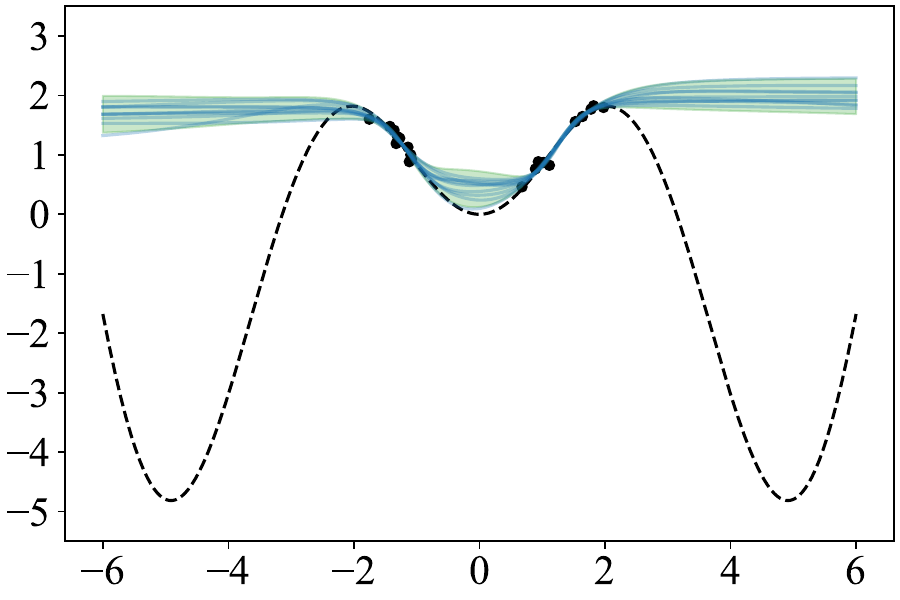}
\begin{center}\vspace{-5pt}\small (c) \ Deep ensemble of 10 DNNs\end{center}\label{fig:10DNN}
\end{minipage}
\\
\begin{minipage}[b]{0.32\linewidth}
\includegraphics[width=\linewidth]{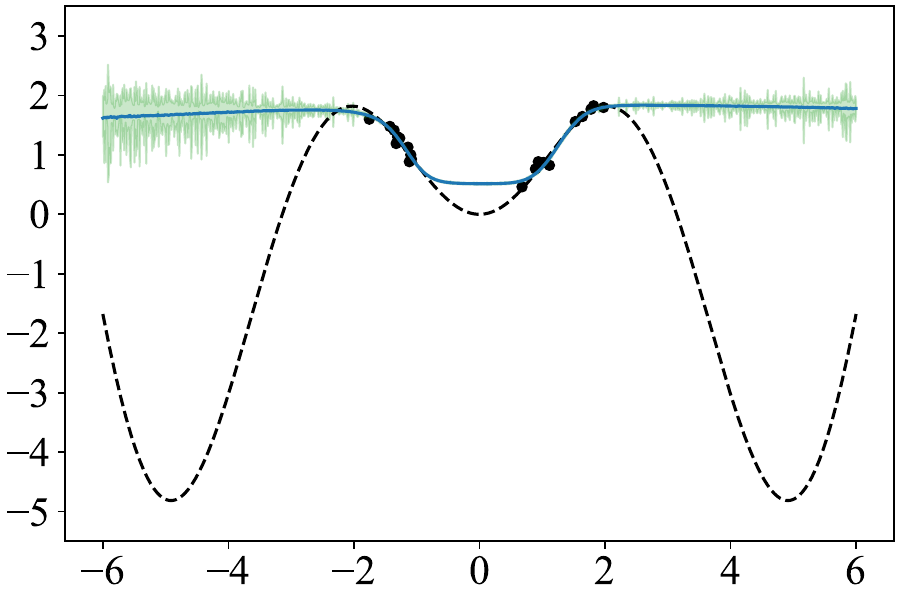}
\begin{center}\vspace{-5pt}\small (d) \ BNN (single Gaussian)\end{center}\label{fig:1BNN}
\end{minipage}
\begin{minipage}[b]{0.32\linewidth}
\includegraphics[width=\linewidth]{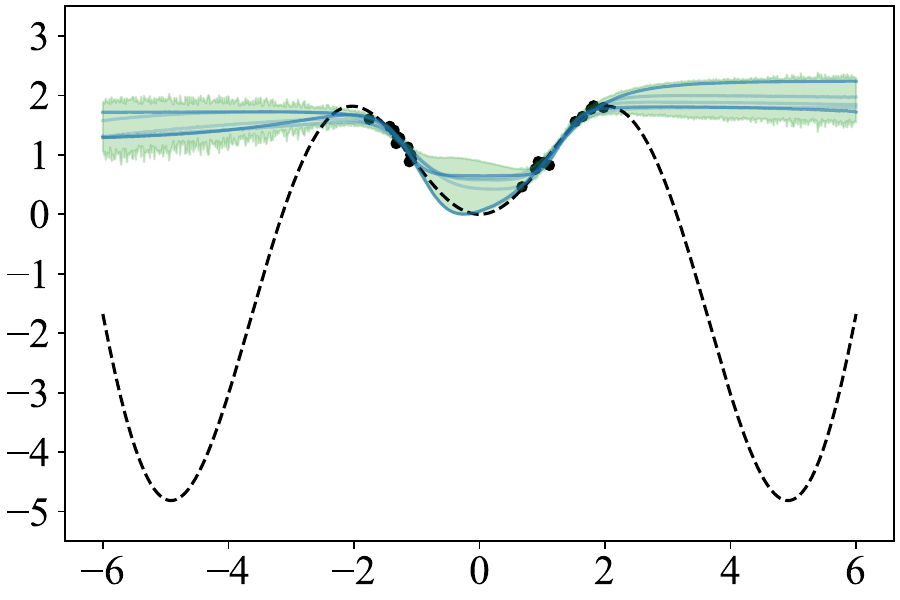}
\begin{center}\vspace{-5pt}\small (e) \ BNN-GM of 5 components\end{center}\label{fig:5BNN}
\end{minipage}
\begin{minipage}[b]{0.32\linewidth}
\includegraphics[width=\linewidth]{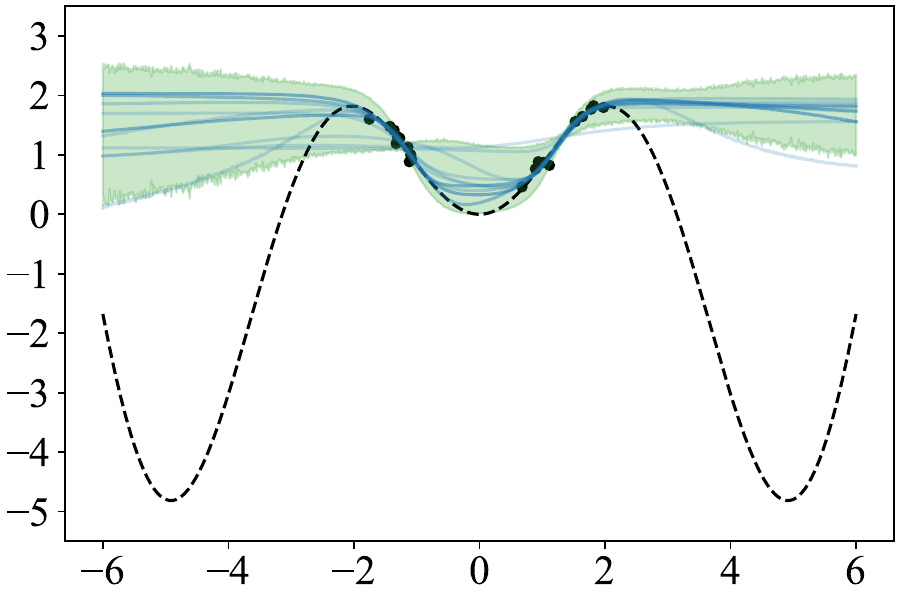}
\begin{center}\vspace{-5pt}\small (f) \ BNN-GM of 10 components\end{center}\label{fig:10BNN}
\end{minipage}
\\
\begin{minipage}[b]{0.32\linewidth}
\includegraphics[width=\linewidth]{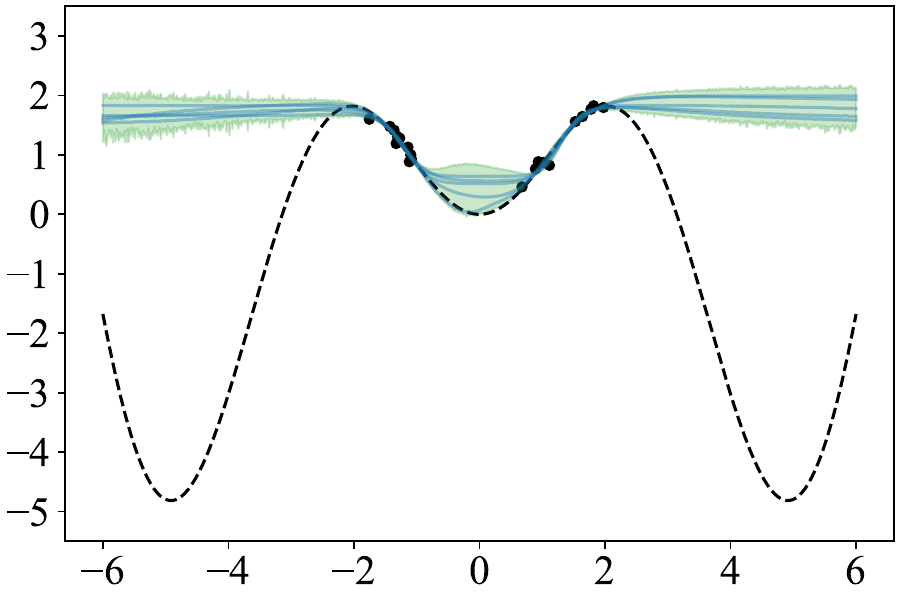}
\begin{center}\vspace{-5pt}\small (g) \ Deep ensemble of 5 BNNs\end{center}\label{fig:5-1BNN}
\end{minipage}
\begin{minipage}[b]{0.32\linewidth}
\includegraphics[width=\linewidth]{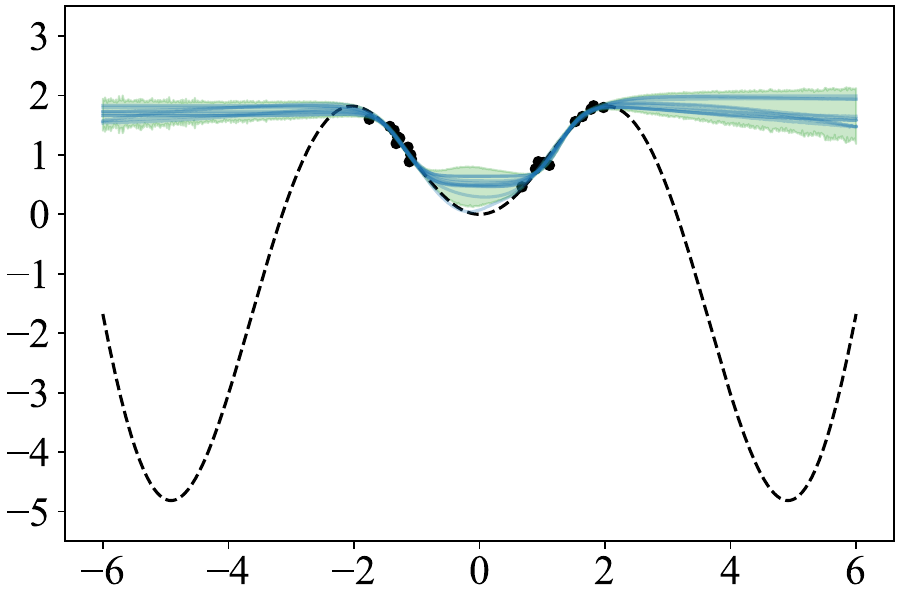}
\begin{center}\vspace{-5pt}\small (h) \ Deep ensemble of 10 BNNs\end{center}\label{fig:10-1BNN}
\end{minipage}
\caption{\small Uncertainty estimation on the toy $1$D regression task.
Each line indicates one component of the ensemble or mixture, and the filled region indicates the mean value of the prediction with the standard deviation$\times(\pm 2)$.
As for BNN-GMs, the stronger the color intensity of the line, the larger the mixture coefficients of the model.
}
\label{fig:bnngm}
\end{figure*}

\subsection{Details of experiment}
\label{app:experiment-bnn}
We give a detailed explanation for the toy $1$D regression experiment in Appendix~\ref{sec:experiment-bnn}.
The task is to learn a curve $y = x \sin(x)$ from a training dataset that consists of $20$ points sampled from the noised curve $y = x \sin(x) + \varepsilon$, $\varepsilon \sim \mathcal{N}(0, 0.1^2)$, see Figure~\ref{fig:bnngm}.
To obtain the regression model of the curve, we used the neural network model as the base model that had two hidden layers and $8$ hidden units in each layer with erf activation.
Regarding the Bayes inference for the BNN-GM, we modeled the prior as $\mathcal{N}(0, \sigma_w)$ and the variational family as the Gaussian mixture.
Furthermore, we chose the likelihood function as the Gaussian distribution:
\begin{align*}
    p(y | f(x; w))
    = \mathcal{N}(f(x;w), \sigma_y).
\end{align*}
Then, we performed the SGVB method based on the proposed ELBO \eqref{approx ELBO}, where the batch size was equal to the dataset size.
Hyperparameters were as follows: epochs $= 100$, learning rate $= 0.05$, $\sigma_w = 10^6$, $\sigma_y = 10^{-2}$.

\end{document}